\documentclass[10pt,journal,compsoc]{IEEEtran}

  \usepackage{cite}

\ifCLASSINFOpdf
  \usepackage[pdftex]{graphicx}
\else

   \usepackage[dvips]{graphicx}

\fi

\usepackage{float}
\usepackage[justification=centering]{caption} 

\usepackage{amsmath,amsthm,amsfonts}
\usepackage{multirow}

\usepackage{algorithm}
\usepackage{algorithmic}
\usepackage{diagbox}
\usepackage{enumitem}

\usepackage[top=0.69in,bottom=0.63in,left=0.5in,textwidth=7.6in]{geometry}
\usepackage{array}
\usepackage[pagebackref=true,breaklinks=true,linkcolor=red,anchorcolor=blue, citecolor=green,letterpaper=true,colorlinks,bookmarks=false]{hyperref}
\usepackage{mdwmath}
\usepackage{mdwtab}
\usepackage{mathrsfs}
\usepackage{dutchcal}
\newtheorem{theorem}{Theorem}
\usepackage{booktabs}
\usepackage{array}

\ifCLASSOPTIONcompsoc
  \usepackage[caption=false,font=footnotesize,labelfont=sf,textfont=sf]{subfig}
\else
  \usepackage[caption=false,font=footnotesize]{subfig}
\fi

\usepackage{stfloats}

\begin{document}

\title{Fast Sparse PCA via Positive Semidefinite\\ Projection for Unsupervised Feature Selection}

\author{Junjing Zheng,
        Xinyu Zhang*,
        Yongxiang Liu,
        Weidong Jiang,
        Kai Huo,
        Li Liu
\IEEEcompsocitemizethanks{\IEEEcompsocthanksitem \textcolor{black}{This work was supported by the National Key Research and Development Program of China No. 2021YFB3100800, the
National Science Foundation of China under Grants 61025006, 60872134, 62376283 and
61901482 and the China Postdoctoral Science Foundation under Grant
2018M633667. The authors are with the College of Electronic Science and Technology,} National University of Defense Technology, Changsha 410073,
China. 
Emails: \{zjj20212035@163.com;zhangxinyu90111@163.com; huokai2001@163.com; jwd2232@vip.163.com;liuli\_nudt@nudt.edu.cn\}
\IEEEcompsocthanksitem 
Corresponding author: Xinyu Zhang}
}

\markboth{Submitted to IEEE Transactions on Pattern Analysis and Machine Intelligence}%
{Zheng \MakeLowercase{\textit{et al.}}: Fast Sparse PCA via Positive Semidefinite Projection for Unsupervised Feature Selection}

\IEEEtitleabstractindextext{%
\begin{abstract}
 In the field of unsupervised feature selection, sparse principal component analysis (SPCA) methods have attracted more and more attention recently. Compared to spectral-based methods, SPCA methods don't rely on the construction of a similarity matrix and show better feature selection ability on real-world data. The original SPCA formulates a nonconvex optimization problem. Existing convex SPCA methods reformulate SPCA as a convex model by regarding the reconstruction matrix as an optimization variable. However, they are lack of constraints equivalent to the orthogonality restriction in SPCA, leading to 
larger solution space. In this paper, it's proved that the optimal solution to a convex SPCA model falls onto the Positive Semidefinite (PSD) cone. A standard convex SPCA-based model with PSD constraint for unsupervised feature selection is proposed. Further, a two-step fast optimization algorithm via PSD projection is presented to solve the proposed model. Two other existing convex SPCA-based models are also proven to have their solutions optimized on the PSD cone in this paper. Therefore, the PSD versions of these two models are proposed to accelerate their convergence as well. We also provide a regularization parameter setting strategy for our proposed method. Experiments on synthetic and real-world datasets demonstrate the effectiveness and efficiency of the proposed methods.
\end{abstract}

}

\maketitle

\IEEEdisplaynontitleabstractindextext

%
\IEEEpeerreviewmaketitle

\ifCLASSOPTIONcompsoc
\IEEEraisesectionheading{\section{Introduction}\label{sec:introduction}}
\else
\section{Introduction}
\label{sec:introduction}
\fi
\IEEEPARstart{F}{e}ature Selection (FS)  is the process of automatically selecting a small subset of relevant features (variables) for an interested learning task \cite{duda2006pattern,A_book_about_FS}. It has been a longstanding and essential research area in machine learning and data mining in the past decades and has a wide range of realistic applications (especially those involving high-dimensional features but comparatively fewer data points~\cite{bioinformatics_FS}), such as image processing \cite{image}, action recognition \cite{Neu}, bioinformatics~\cite{bioinformatics}, \emph{etc}. The central premise underlying FS is that data features contain redundant or irrelevant features that can thus be discarded without resulting in significant information loss \cite{Review_UFS_2}. The benefits of applying FS \cite{Review_UFS} include compact features, computational efficiency, mitigating the problem of overfitting, possibly better performance, facilitating data visualization and interpretation, \emph{etc}. According to the amounts of supervised information required for the target task during training, FS methods can be classified as supervised \cite{supervised-FS}, semisupervised \cite{Semi-supverised-FS} and unsupervised \cite{unsupervised-FS}. Among them, supervised FS has been extensively studied. However, it requires the training data to be fully labeled, which raises serious limitations for some practical applications. Therefore, Unsupervised Feature Selection (UFS) \cite{IFS,SPCAFS,tensor-FS} without the need for any human annotations has gained increasing attention in recent years.


UFS methods can be classified into three categories, namely filter, wrapper, and embedded methods~\cite{FS-categories}. Wrapper methods \cite{wrapper} evaluate the relevance of a feature subset through actually learning a pretext task (\emph{e.g.}, clustering) with it and search for the optimal feature subset with the best performance. Wrapper methods can find good feature subsets but are computationally expensive. By contrast, filter methods \cite{filter_1,filter_2} often exploit data intrinsic properties as criteria to directly score each feature without repeated iterations and are usually more efficient than wrapper methods. However, filter methods usually have relatively poor performance. Embedded methods simultaneously perform feature selection and learning algorithms within one optimization problem. Thanks to the embedding idea, embedded methods have the natural advantage of utilizing various machine learning principles. Therefore, embedded methods are more popular than the other two categories because of their effectiveness and rather low computational complexity. 
The most widely-studied embedded methods are spectral-based methods \cite{LS, UDFS, NDFS, JELSR, OCLSP, SOGFS}, which employ graph techniques to describe data local structure and utilize spectral analysis tools such as spectral regression to learn a feature weight matrix for FS. In recent years, Sparse Principal Component Analysis (SPCA) \cite{SPCA} has become \cite{CSPCA, AW-SPCA, SPCAFS} a new technique for Embedded UFS (EUFS) due to its advantages including simplicity yet elegance, relatively computational efficiency, inheriting the global manifold learning property of PCA \cite{PCA}, and less sensitive to noise features than spectral methods.
In this paper, we are interested in Embedded UFS (EUFS).

However, EUFS methods still have the following limitations. Firstly, most EUFS methods focus more on the local manifold structure of data with the construction of similarity matrix \cite{LS, NDFS, UDFS, SOGFS, DCUFS, REM-FS}, while the global manifold structure is not fully exploited. This inevitably leads to sensitivity to noise features and outlier data. 
Secondly, many EUFS methods construct a non-convex optimization problem with a series of subproblems to solve in each iteration \cite{OCLSP, JELSR, DRSI-FS,DHLFS}. Considering the computational complexity of them is usually cubic to the number of samples (or the number of features), the running time of solving the whole problem can sometimes be unacceptable. 
Therefore, it is important to propose an effective and efficient convex EUFS model with as few subproblems as possible. Finally, almost all EUFS methods lack a regularization parameter-setting strategy to guide practice. For the third problem, to the best of our knowledge, there are no UFS methods that clearly provide a regularization parameter-setting strategy. Most papers only vaguely analyze the sensitivity of the proposed method to each parameter without telling how to select a proper combination of parameters 
 \cite{CSPCA, AW-SPCA, SPCAFS, SOGFS, WPANFS}. More time is required for these methods in parameter tuning.

In conclusion, The ideal EUFS method to address the above problems should 1) provides good preservation of global data structure; 2) 
has a simple model which includes as few subproblems as possible, avoiding iterations inside iterations; 3) has a convex optimization problem so that good convergence is guaranteed; 4) has a computational complexity that is influenced by the number of samples as small as possible in each iteration. 5) provides a regularization parameter setting strategy. To the best of our knowledge, convex SPCA-based EUFS methods are the most proper choice to meet the above requirements. To this day, there are already several convex SPCA models designed for UFS\cite{CSPCA}\cite{AW-SPCA}. They outperform other EUFS methods according to the experimental results. But as we will see later, they are for now not fast enough due to the large solution space they construct. To this end, we propose a standard convex SPCA-based EUFS method, and theoretically prove that for existing convex SPCA-based EUFS methods, the true solution space is the positive semidefinite cone (PSD cone). With PSD projection, we come out with a fast optimization algorithm for our proposed model, and successfully accelerate other convex SPCA-based EUFS methods while maintaining their FS ability. Our contributions can be summarized as follows:
\begin{itemize}[leftmargin=*]
    \item A standard convex SPCA optimization problem for UFS is proposed. We reformulate SPCA as a convex model by regarding the reconstruction matrix as an optimization variable. A $\ell_{2,1}$-norm and a nuclear norm are then incorporated into the model to make the reconstruction matrix sparse and low-rank respectively. We prove that the optimal solution of the reconstruction matrix falls onto the PSD cone, and add a PSD constraint to the optimization problem. We name the proposed method \emph{SPCA-PSD}.
    \item A two-step fast optimization algorithm is presented to solve the proposed model.  Due to the PSD constraint, the nuclear norm is simplified into a trace function term of the reconstruction matrix. And the algorithm updates the reconstruction matrix by first utilizing the derivative of the objective function and then performing PSD projection. We also give a convergence analysis on the algorithm. The experimental results on both synthetic and real-world data show that the proposed algorithm converges fast with the PSD projection and provides good feature selection performance.
    \item We prove that for other existing convex SPCA-based EUFS models \emph{CSPCA}\cite{CSPCA} and \emph{AW-SPCA}\cite{AW-SPCA}, the PSD constraint still holds. So we add the PSD projection into the original algorithms of them and propose their PSD versions: \emph{CSPCA-PSD} and \emph{AW-SPCA-PSD}. Several experiments using real-world data are conducted to prove the enhancement that the PSD projection achieves. 
    \item We provide a regularization parameter-setting strategy for SPCA-PSD by analyzing the experimental results. With this strategy, one can easily tune the parameters to obtain both effectiveness and efficiency. To the best of our knowledge, it is the first time for a UFS method to propose a clear instruction on how to choose the optimal combination of regularization parameters.
\end{itemize}

The rest of the paper is organized as follows. In Section 2, we introduce the related work. In Section 3, we propose a standard convex SPCA-based EUFS method and put forward a two-step fast optimization algorithm to solve the proposed model, In Section 4, we give the convergence analysis and the computational complexity of SPCA-PSD, and discuss the relationship between SPCA-PSD and other SPCA-based UFS methods. Then, we came up with two more proposed methods: CSPCA-PSD and AW-SPCA-PSD. In Section 5, we conduct experiments on both synthetic and real-world data to evaluate the effectiveness and efficiency of the proposed method. In Section 6, we make a conclusion. In \emph{Appendix A}, we provide proof of how PSD constraint holds in CSPCA and AW-SPCA.

\section{Related work}
In this section, after describing some notations used in this paper, we first have a review of spectral-based UFS methods, discussing both their achievements and drawbacks, then we describe the basic principle of PCA and SPCA, and further introduce the new type of EUFS methods based on SPCA. 

\subsection{Notations}
We first introduce notations that will be used throughout this paper. We use bold capital characters to denote matrices, bold lowercase characters to denote vectors. We denote symmetric cone as $S^d$ and positive semidefinite cone as $S_+^d$, $d\in\mathbb{N_+}$. Given an arbitrary matrix $\mathbf{A}\in\mathbb{R}^{m\times{n}}$ ($m$ and $n$ represent the number of rows and columns respectively), the $(i,j)$-th entry of $\mathbf{A}$ is denoted by $a_{ij}$, its $i$-th row vector and $j$-th column vector are denoted by $\mathbf{a}^i$, $\mathbf{a}_j$, respectively. The transposition of $\mathbf{A}$ is denoted by $\mathbf{A}^T$. If $\mathbf{A}$ is a square matrix, its trace is denoted by $Tr(\mathbf{A})$, and a vector that contains the diagonal elements of $\mathbf{A}$ is denoted by $diag(\mathbf{A})$. We denote the rank of $\mathbf{A}$ as $rank(\mathbf{A})$. For two matrices $\mathbf{A}$ and $\mathbf{B}$ that have the same size, we denote their Hadamard product by $\mathbf{A}*\mathbf{B}$. Finally, the nuclear norm and $\ell_{2,1}$-norm are defined as follows respectively
\begin{equation}
    \Vert \mathbf{A} \Vert_*=\sum\limits_{j=1}^{n}\phi_j,
\end{equation}
\begin{equation}
    \Vert \mathbf{A}\Vert_{2,1}=\sum\limits_{j=1}^{n}\sqrt{\sum\limits_{i=1}^{m}a^2_{ij}},
\end{equation}
where $\phi_j$ denotes the $i$-th singular value of $\mathbf{A}$.
It’s important to note that the definition of $\ell_{2,1}$-norm in this paper is opposite to those in most of the papers, for the convenience of deduction. We calculate $\ell_{2}$-norm of the column vectors in a matrix, then sum them to obtain its $\ell_{2,1}$-norm.

\subsection{Spectral-based Embedded Unsupervised Feature Selection Method}
There are two major characteristics shared by most spectral-based EUFS method : 1) calculation of a similarity matrix, 2) low-dimensional representation, which enable them to learn the local manifold structure of data and the correlation between features. However, to some degree, the comprehensive performance of spectral-based EUFS methods are limited by the above two properties. 

The calculation of a similarity matrix is realized by first constructing a graph. And spectral-based methods can be roughly divided into two types\cite{OCLSP}: one to use a predefined graph, while the other learn an adaptive graph during the optimization. 
\begin{itemize}[leftmargin=*]
\item \textbf{Predefined graph}. For predefined-graph-based methods such as LapScore\cite{LS}, MCFS\cite{MCFS}, NDFS\cite{NDFS}, JELSR\cite{JELSR}, RUFS\cite{RUFS}, RSFS\cite{RSFS}, etc., after constructing a k nearest neighborhood (knn) graph that connects each sample with its neighbors, a similarity matrix, which utilizes certain criterion to measure pairwise sample similarity, will be calculated and fixed during the optimization. The knn graph is expected to connect samples that belong to the same class and represent the local geometrical structure of data. However, real-world data tends to contain noise features, which can misguide the graph to come out with false connections. Therefore the optimization may be based on an unreliable foundation and ends up with a sub-optimal feature subset. 
\item \textbf{Adaptive graph}. Methods like SOGFS\cite{SOGFS}, DCUFS\cite{DCUFS}, WPANFS\cite{WPANFS}, FSASL\cite{FSASL} and OCLSP\cite{OCLSP}, etc., treat the similarity matrix as a variable to obtain an adaptive graph. They outperform earlier spectral-based methods on some data sets due to the flexibility. The price, however, is that the computational complexity of each iteration is related to the square (sometimes even cubic) of the number of samples. When facing  big data set, these methods could cost much running time. 
\end{itemize}

Although most spectral-based methods may use transformation matrix to project each sample into a low-dimensional space\cite{SOGFS}\cite{NDFS}\cite{OCLSP}, they usually ignore whether the obtained transformation matrix can reconstruct the data matrix from the subspace properly or not. As a result, they don't fully preserve the global manifold structure of data, losing some significant information. Recently, there is a growing interest in treating the reconstruction error as a criterion  in feature selection task\cite{REFS}\cite{RNE}\cite{CSPCA}\cite{SPCAFS}, including SPCA-based methods. 

\subsection{SPCA-based Embedded Unsupervised Feature Selection Method}
\subsubsection{PCA and SPCA}
Principal Component Analysis (PCA)\cite{PCA} is a well-known unsupervised feature extraction method that has been studied successively for decades \cite{SPCA, RPCA, OMRPCA, MSPCA}. Its core idea is to find an orthogonal transformation matrix that can perform a linear combination of the original features, projecting them into a new feature space. Part of the features in this new space capture maximal variance among samples \cite{MPCA}. In this way, these features are called "principal components (PCs) " as they preserve most information of the original data.

Given a data matrix $\mathbf{X}=\left[\mathbf{x}_1,\mathbf{x}_2,\mathbf{x}_3,\cdots,\mathbf{x}_n\right]\in\mathbb{R}^{d\times{n}}$ , where each column vector $\mathbf{x}_i\in\mathbb{R}^{d\times{1}}$ represents a sample with $d$ features. Here we assume that $\mathbf{X}$ has been centralized, which means each $\mathbf{x}_i$ is replaced by $\mathbf{x}_i-\frac{1}{n}\sum\limits_{i=1}^{n}\mathbf{x}_i$ so that the center of $\mathbf{X}$ will be moved to the zero point. Under the above assumption, PCA can be conducted by solving  
\begin{equation}
    \min\limits_{\mathbf{U}^T\mathbf{U}=\mathbf{I}_k} {\Vert{\mathbf{X}-\mathbf{U}\mathbf{U}^{T}\mathbf{X}}\Vert}^2_F. \label{PCA}
\end{equation}
where $\mathbf{I}_k$ is an $k$-dimensional identity matrix, $\mathbf{U}\in\mathbb{R}^{d\times{k}}$ ($k \ll d$) is the transformation matrix that contains $k$ the combination weights (so-called 'loadings') of the original features. $\mathbf{U}^T$ projects each $\mathbf{x}_i$ into a $k$-dimensional feature space, while $\mathbf{U}$ projects them back to a $d$-dimensional space. There we can view $\mathbf{U}\mathbf{U}^T\mathbf{X}$ as a reconstruction of $\mathbf{X}$. Since Problem (\ref{PCA}) aims to minimize the reconstruction error, it ensures that the learned transformation matrix $\mathbf{U}$ can preserve as much \emph{global information} of $\mathbf{X}$ as possible. Thus, PCA is a useful global manifold structure learning method.

The motivation for modifying PCA into \emph{SPCA} lies in a drawback of PCA \cite{SPCA}: Although $\mathbf{U}^T\mathbf{X}$ reduces the dimension of each $\mathbf{x}_i$, the loadings in $\mathbf{U}^T$ are typically nonzero. When the combination weights of noise features are not small enough to be ignored, yet they still take part in forming the new feature space, then both interpretability and performance are not guaranteed. To solve this problem, researchers sought a way to produce sparse combination weights. As is widely known, LASSO \cite{lasso} is a penalized least squares regression model that imposes a $\ell_{1}$-norm (referred to as LASSO penalty) on regression coefficients to gain sparsity in the linear regression model. And Elastic Net \cite{elastic_net} is proposed as an improved version of LASSO by adding ridge penalty ($\ell_{2}$-norm) to overcome the limitation of being unsuitable to data where $d\gg{n}$. In order to obtain sparse loadings for PCA, SPCA \cite{SPCA} was proposed by formulating a self-contained type of Elastic Net, formulated as follows:
\begin{equation}
\begin{aligned} \label{SPCA}
    \min\limits_{\mathbf{U},\mathbf{Q}} \quad &{\Vert {\mathbf{X}-\mathbf{U}\mathbf{Q}^T\mathbf{X}} \Vert}^2_F+\alpha\sum\limits_{j=1}^k{\Vert {\mathbf{q}_j}\Vert}^2_2+\beta\sum\limits_{j=1}^k{\Vert {\mathbf{q}_j}\Vert}_1\\
    s.t. \quad&{\mathbf{U}^T\mathbf{U}=\mathbf{I}_k},
\end{aligned}
\end{equation}
where $\mathbf{Q}\in\mathbb{R}^{d\times{k}}$ is the regression coefficients matrix, ${\mathbf{q}_j}$ represents the $j$-th column vector of $\mathbf{Q} (j=1,2,\cdots,k)$, and $\alpha$, $\beta$, are both regularization parameters. Note that the original Elastic Net measures least squares after $\mathbf{X}$ is transformed (for example, $\Vert\mathbf{Z}-\mathbf{Q}^T\mathbf{X} \Vert^2_F$, where $\mathbf{Z}\in\mathbb{R}^{k\times{n}}$ ). So the description "self-contained" means that Problem (\ref{SPCA}) measures least squares in the original feature space of $\mathbf{X}$, which is just a different angle of explaining "reconstruction". Thus, $\mathbf{Q}$ can now treated as a transformation matrix. In this way, Problem (\ref{SPCA}) builds a bridge between Problem (\ref{PCA}) and Elastic Net. In fact, if we restrict $\mathbf{Q}=\mathbf{U}$ and set $\alpha=\beta=0$, then Problem (\ref{SPCA}) will be degenerated into Problem (\ref{PCA}), which means SPCA can be viewed as a relaxation of PCA. 

Interestingly, SPCA not only constructs a bridge between PCA and Elastic Net, but also enables PCA to conduct unsupervised feature selection. The self-contained term is able to fully learn the global manifold structure of data by minimizing the reconstruction error. The LASSO penalty creates the precondition of feature selection because, during the process of optimization, the learned transformation matrix is automatically forced to have sparse loadings, implying the importance of each feature that takes part in the reconstruction. And obviously, SPCA has no need to construct a similarity matrix that is sensitive to noise feature. 

\subsubsection{Existing SPCA-based EUFS methods}
With the introduction of PCA and SPCA above, it is natural to consider proposing SPCA-based EUFS methods. To the best of our knowledge, there exists three SPCA-based EUFS methods for now: SPCAFS\cite{SPCAFS}, CSPCA\cite{CSPCA} and AW-SPCA\cite{AW-SPCA}.  SPCAFS\cite{SPCAFS} directly incorporates $\ell_{2,p}$-norm of the projection matrix into the original PCA problem (\ref{PCA}) and obtains a rather simple model. Although SPCAFS is simple enough, it possesses a non-convex optimization problem and has a possibility of converging to a sub-optimal solution.  Both CSPCA and AW-SPCA exploit the idea of regarding $\mathbf{U}\mathbf{Q}^T$ as a new optimization variable called 'reconstruction matrix' to formulate convex SPCA models, obtaining stable convergence and good feature selection ability. However, they both ignore the original orthogonality constraint of $\mathbf{U}$. So they come out with an unconstrained optimization problem and design an optimization algorithm using only the derivation of the objective function, resulting in more iterations to search for the optimal solution. Sometimes the incorrect solution space may lead to performance penalty. In this paper, we also reformulate the original SPCA as a convex model by regarding the reconstruction matrix as an optimization variable. And we prove that the optimal solution falls onto the positive semidefinite (PSD) cone. We also prove that the PSD constraint holds for CSPCA and AW-SPCA as well. By finding the true solution space, we are able put forward a fast SPCA-based UFS method and accelerate existing convex SPCA-based methods.

\section{Standard Convex SPCA-based Model for Unsupervised Feature Selection}
In this section, we first construct a standard convex SPCA-based model for unsupervised feature selection. Then we propose an algorithm via derivative and PSD projection. We name the whole method \emph{SPCA-PSD}.
\subsection{Model}
In the field of feature selection, $\ell_{2,1}$-norm is usually used to constrain the sparsity of a matrix due to its convenient optimization. Since the sum of each vector's $\ell_{2}$-norm is calculated in $\ell_{2,1}$-norm, it can also to some degree achieve the same goal of the $\ell_{2}$-norm in SPCA. Therefore, it can reduce the number of regularization parameters, which is beneficial to parameter tuning. By replacing the regularization terms in (\ref{SPCA}) with $\Vert \mathbf{Q}^T \Vert_{2,1}$, we obtain the following problem:
\begin{equation}
    \begin{aligned} \label{Problem 6}
    \min\limits_{\mathbf{U},\mathbf{Q}} \quad &{\Vert {\mathbf{X}-\mathbf{U}\mathbf{Q}^T\mathbf{X}} \Vert}^2_F+\lambda{\Vert \mathbf{Q}^T \Vert}_{2,1}\\
    s.t. \quad&{\mathbf{U}^T\mathbf{U}=\mathbf{I}_k},
\end{aligned}
\end{equation}
where $\lambda>0$.

In \cite{SPCA}, it has been proven that if given the eigenvalue decomposition  $\mathbf{X}\mathbf{X}^T\left(\mathbf{X}\mathbf{X}^T+\lambda\right)^{-1}\mathbf{X}\mathbf{X}^T=\mathbf{V}\mathbf{\Sigma}\mathbf{V}^T$, then $\mathbf{q}_j^*\propto\mathbf{v}_j$ in Problem (\ref{SPCA}). Specifically, we have $\mathbf{q}_j^*=s_j\frac{\sigma^2_{jj}}{\sigma^2_{jj}+\alpha}\mathbf{v}_j$, ($s_j=1$ or $-1$), and $\mathbf{u}_j=s_j\mathbf{v}_j$. It can be further inferred that $\mathbf{U}^*(\mathbf{Q}^*)^T=\mathbf{V}\mathbf{D}\mathbf{V}^T$, where $d_{jj}=\frac{\sigma^2_{jj}}{\sigma^2_{jj}+\alpha} \geq 0$. Thus, $\mathbf{U}^*(\mathbf{Q}^*)^T\in{S^d_+}$. In this paper, we argue that the conclusion of  $\mathbf{U}^*(\mathbf{Q}^*)^T\in{S^d_+}$ still holds in Problem (\ref{Problem 6}). 
\begin{theorem} \label{theorem 1}
Let $\mathbf{U}^*$ and $\mathbf{Q}^*$ be the optimal solution to Problem (\ref{Problem 6}), Then 
$ \mathbf{U}^*(\mathbf{Q}^*)^T\in{S^d_+} $
\end{theorem}
\begin{proof}
We can rewrite Problem (\ref{Problem 6}) as 
\begin{equation} \label{Problem 7}
\begin{aligned}
 &\min\quad{\left\Vert {\mathbf{X}-\sum\limits_{j=1}^{k}\mathbf{u}_j\mathbf{q}_j^T\mathbf{X}}\right\Vert}^2_F+\lambda\sum\limits_{j=1}^{k}{\Vert {\mathbf{q}_j}\Vert}_2\\
&s.t.\quad \mathbf{u}_j^T\mathbf{u}_j=1.
\end{aligned}
\end{equation}
Then the objective function can be expanded as 
\begin{equation} \label{Problem 8}
\begin{aligned}
&{\left\Vert {\mathbf{X}-\sum\limits_{j=1}^{k}\mathbf{u}_j\mathbf{q}_j^T\mathbf{X}}\right\Vert}^2_F+\lambda\sum\limits_{j=1}^{k}{\Vert {\mathbf{q}_j}\Vert}_2\\
&=Tr(\mathbf{X}\mathbf{X}^T)-\sum\limits_{j=1}^{k}\bigg[2Tr(\mathbf{u}_j^T\mathbf{X}\mathbf{X}^T\mathbf{q}_j)-Tr(\mathbf{q}_j^T\mathbf{X}\mathbf{X}^T\mathbf{q}_j)\bigg.\\
&-\left.\frac{\lambda}{\Vert \mathbf{q}_j\Vert_2}\mathbf{q}_j^T\mathbf{q}_j\right]\\
&=Tr(\mathbf{X}\mathbf{X}^T)-\sum\limits_{j=1}^{k}\bigg[2(\mathbf{u}_j^T\mathbf{X}\mathbf{X}^T\mathbf{q}_j)-(\mathbf{q}_j^T\mathbf{X}\mathbf{X}^T\mathbf{q}_j)\bigg.\\
&+\left.\frac{\lambda}{\Vert \mathbf{q}_j\Vert_2})\mathbf{q}_j^T\mathbf{q}_j\right].
\end{aligned}
\end{equation}
If we view (\ref{Problem 8}) as a sum of $k$ subproblems with respect to $\mathbf{q}_j$ and $\mathbf{u}_j$, then given a fixed $\mathbf{u}_j$, we can have each subproblem minimized at
\begin{equation} \label{Problem 9}
   \mathbf{q}_j^*=\left(\mathbf{X}\mathbf{X}^T+\frac{\lambda}{\Vert \mathbf{q}_j\Vert_2}\right)^{-1}\mathbf{X}\mathbf{X}^T\mathbf{u}_j.
\end{equation}
Although the expression of $\mathbf{q}_j^*$ contains ${\Vert \mathbf{q}_j\Vert_2}$, we can update $\mathbf{q}_j$ by using ${\Vert \mathbf{q}_j\Vert_2}$ from the last iteration during optimizing process. Therefore, $\frac{\lambda}{\Vert \mathbf{q}_j\Vert_2}$ is a constant in each iteration.

By substituting (\ref{Problem 9}) back to each subproblem in (\ref{Problem 8}), we can obtain the following equation:
\begin{equation}
    \mathbf{u}_j^*={\underset{\mathbf{u}_j^T\mathbf{u}_j=1}{\arg\min}}\mathbf{u}^T\mathbf{X}\mathbf{X}^T\left(\mathbf{X}\mathbf{X}^T+\frac{\lambda}{\Vert \mathbf{q}_j \Vert_2}\right)^{-1}\mathbf{X}\mathbf{X}^T\mathbf{u}_j.
\end{equation}\par
It can be solved by performing an eigenvalue decomposition: $\mathbf{X}\mathbf{X}^T\left(\mathbf{X}\mathbf{X}^T+\frac{\lambda}{\Vert \mathbf{q}_j \Vert_2}\right)^{-1}\mathbf{X}\mathbf{X}^T=\mathbf{V}\mathbf{\Sigma}\mathbf{V}^T$. Hence $\mathbf{u}_j^*=s_j\mathbf{v}_j$ with $s_j=1$ or $-1$. Then, we obtain $\mathbf{q}_j^*=s_j\frac{\sigma^2_{jj}}{\sigma^2_{jj}+{\lambda}/{\Vert \mathbf{q}_j\Vert_2}}\mathbf{v}_j$. Finally, we have
\begin{equation}\label{Cls 1}
    \mathbf{U}^*{(\mathbf{Q}^*)}^T=\sum\limits_{j=1}^{k}\mathbf{u}_j^*(\mathbf{q}^*)^T_j=\mathbf{V}\mathbf{D}\mathbf{V}^T\in{S^d_+},
\end{equation}
where $d_{jj}=\frac{\sigma^2_{jj}}{\sigma^2_{jj}+{\lambda}/{\Vert \mathbf{q}_j\Vert_2}} \geq 0$. For each iteration, $\mathbf{U}^*{(\mathbf{Q}^*)}^T$ satisfies (\ref{Cls 1}). Thus, (\ref{Cls 1}) still holds when Problem (\ref{Problem 7}) is solved. 
\end{proof}

In \textbf{Appendix A}, we prove that for existing convex SPCA-based models for UFS: CSPCA \cite{CSPCA} and AW-SPCA\cite{AW-SPCA}, the PSD constraint still holds. From a certain perspective, CSPCA and AW-SPCA can be viewed as variations of our model, which will be discussed later in \textbf{Section 4}. The significant difference between the proposed model and these two models is that we discover the true solution space for the convex formulation of SPCA's optimization problem that treats reconstruction matrix as variable. The rest of this section will show how this significant difference allows us to construct a convex model and design an optimization algorithm that is fast both in theory and in practice. 

Apparently, Problem (\ref{Problem 6}) is still a non-convex problem. Although we can update the two optimization variables alternately in each iteration to solve the problem, it could still not be efficient enough in practice. Thus, we need to reformulate Problem (\ref{Problem 6}) into a convex one. Firstly, considering the column-orthogonality of $\mathbf{U}$, it can be easily proved that 
${\Vert \mathbf{Q}^T \Vert_{2,1}}$ is equal to ${\Vert \mathbf{U}\mathbf{Q}^T \Vert_{2,1}}$. Secondly, we replace $\mathbf{U}\mathbf{Q}^T$ with a new variable $\mathbf{\Omega}\in\mathbb{R}^{d\times{d}}$, which can be viewed as a reconstruction matrix. Thirdly, according to \textbf{Theorem} \textbf{\ref{theorem 1}}, we add PSD constraint on $\mathbf{\Omega}$. Since $\mathbf{U}\in\mathbb{R}^{d\times{k}}$ and $\mathbf{Q}\in\mathbb{R}^{d\times{k}}$, the rank of $\mathbf{\Omega}$ is no greater than $k$. Therefore, we get
\begin{equation}
    \begin{aligned}
            \min\quad &{\Vert {\mathbf{X}-\mathbf{\Omega}\mathbf{X}} \Vert}^2_F+\lambda{\Vert \mathbf{\Omega} \Vert}_{2,1}\\
    s.t. \quad&{\mathbf{\Omega}\in{S^d_+}}\\
    &{rank(\mathbf{\Omega}) \leq k}.
    \end{aligned}
\end{equation}\par
As is all known, $rank(\mathbf{\Omega})$ is non-convex and is hard to solve directly. Hence we use nuclear norm $\Vert \mathbf{\Omega}\Vert_*$ instead, which is a convex approximation to the rank function. Then, we obtain the following optimization problem:
\begin{equation}
\begin{aligned}
    \min\limits_{\mathbf{\Omega}}\quad &{\Vert {\mathbf{X}-\mathbf{\Omega}\mathbf{X}} \Vert}^2_F+\lambda{\Vert \mathbf{\Omega} \Vert}_{2,1}+\eta{\Vert \mathbf{\Omega}\Vert_*}\\
    s.t. \quad&{\mathbf{\Omega}\in{S^d_+}}\\
\end{aligned}
\end{equation}
where $\lambda,\eta>0$. However, the optimization of nuclear norm can still be a challenge. Fortunately, $\mathbf{\Omega}$ is a PSD matrix, which means the following equation holds:
\begin{equation}
    \Vert \mathbf{\Omega} \Vert_*=\sum\limits_{j=1}^{n}\phi_j=\sum\limits_{j=1}^{n}\sigma_j=Tr(\mathbf{\Omega}),
\end{equation}
where $\sigma_j$ is the $i$-th eigenvalue of $\mathbf{\Omega}$. By this equation, we can replace $\Vert \mathbf{\Omega} \Vert$ with $Tr(\mathbf{\Omega})$, significantly reducing the difficulty of optimization and the computation complexity of calculating the objective function. Then finally, we construct the optimization problem of SPCA-PSD as:
\begin{equation}\label{OBJ}
\begin{aligned}
    \min\limits_{\mathbf{\Omega}}\quad &{\Vert {\mathbf{X}-\mathbf{\Omega}\mathbf{X}} \Vert}^2_F+\lambda{\Vert \mathbf{\Omega} \Vert}_{2,1}+\eta{Tr(\mathbf{\Omega})}\\
    s.t. \quad&{\mathbf{\Omega}\in{S^d_+}}.
\end{aligned}
\end{equation}
Note that we obtain this model directly from the original SPCA problem, therefore we view it as the standard convex SPCA-based model for UFS. 

After we obtain the optimal solution $\mathbf{\Omega}^*$, we can score each feature in the data matrix by calculating $\Vert \boldsymbol{\omega}_j\Vert_2$, then sort them in a descending order and select the top-ranked ones. Because $\mathbf{\Omega}^*$ can be viewed as a reconstruction matrix and is sparse in row, each element in $\Vert \boldsymbol{\omega}^i\Vert_2(i=1,2,\cdots,d)$ represents the importance of corresponding feature in $\mathbf{X}$. Thus, it is reasonable to directly use $\Vert \boldsymbol{\omega}_j\Vert_2$  as the score of the $j$-th feature.

\subsection{Optimization Algorithm}
Now, we present an optimization algorithm to solve Problem (\ref{OBJ}). Generally, we combine derivative with PSD projection to search the optimal solution on the PSD cone. In each iteration, we first ignore the PSD constraint and calculate the unconstrained solution by utilizing the derivative of the objective function. Then, we exploit the PSD projection operator to find the nearest solution on the PSD cone. The whole progress can be visualized like Fig.\ref{Algorithm chart}.
\begin{figure}[!t]
    \hspace{-0.3cm}\includegraphics[width=4in]{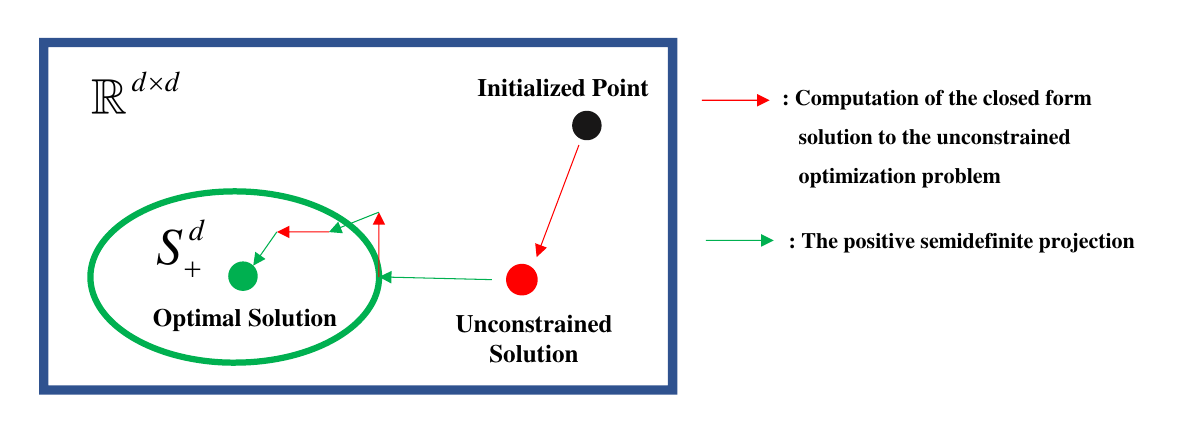}
    \caption{How our proposed algorithm works. Starting from the initialized point, we repeat two steps: 1) Compute unconstrained solution; 2) Project the obtained temporary solution to the nearest point on the PSD cone. In the first few iterations, the unconstrained solution may not be a PSD matrix. But gradually the solutions calculated by both steps will be completely in the solution space. Finally the reconstruction matrix will converge to the optimal solution. We recommend to set the initialized point directly in the PSD cone as to reduce number of iterations.}
    \label{Algorithm chart}
\end{figure}

\subsubsection{The closed form of unconstrained solution}
In order to obtain the closed form solution to the unconstrained optimization problem, we need to calculate the derivative of the objective function. First, we rewrite the problem as:
\begin{equation}\label{Problem 14}
\begin{aligned}
    \min\limits_{\mathbf{\Omega}}\quad &Tr(\mathbf{\Omega}\mathbf{S}\mathbf{\Omega}^T)-2Tr(\mathbf{S}\mathbf{\Omega})+\lambda Tr(\mathbf{\Omega}\mathbf{W}\mathbf{\Omega}^T)+\eta Tr(\mathbf{\Omega})\\
    s.t.\quad&\mathbf{\Omega}\in{S^d_+}.
\end{aligned}
\end{equation}
where $\mathbf{S}=\mathbf{X}\mathbf{X}^T$ and $\mathbf{W}\in{\mathbb{R}^{d\times{d}}}$ is a diagonal matrix whose $j$-th diagonal element is calculated by $\left(1/\left(2\sqrt{\boldsymbol{\omega}_j^T\boldsymbol{\omega}_j+\epsilon_1}\right)\right)$. Note that $\epsilon_1>0$ is a small positive number added to prevent $\sqrt{\boldsymbol{\omega}_j^T\boldsymbol{\omega}_j}$ from being zero.

In one iteration, with $\mathbf{W}$ fixed, take the derivative of the objective function in Problem (\ref{Problem 14}) and set its value equal to zero. Then we can get
\begin{equation}
    \mathbf{\Omega}=\left(\mathbf{S}-\frac{\eta}{2}\mathbf{I}_d\right)\left(\mathbf{S}+\lambda{\mathbf{W}}\right)^{-1}.
\end{equation}
In practice, $\left(\mathbf{S}+\lambda{\mathbf{W}}\right)$ can sometimes be irreversible, especially when $\lambda$ is relatively small. Hence, we improve the robustness of matrix inversion by using the following equation instead:
\begin{equation} \label{derivative}
    \mathbf{\Omega}=\left(\mathbf{S}-\frac{\eta}{2}\mathbf{I}_d\right)\left(\mathbf{S}+\lambda{\mathbf{W}}+\epsilon_2\mathbf{I}_d\right)^{-1},
\end{equation}
where $\epsilon_2$ is also a small positive number. In (\ref{derivative}), the computational complexity of $\left(\mathbf{S}+\lambda{\mathbf{W}}+\epsilon_2\mathbf{I}_d\right)^{-1}$ is $O(d^3)$. In some applications where the number of features is far greater than the number of samples, the computational cost may be too high. Inspired by \cite{CDLFS}, we utilize the Woodbury matrix identity: 
\begin{equation}
\begin{aligned}
   & \left(\mathbf{A}+\mathbf{B}\mathbf{C}\mathbf{D}\right)^{-1}\\
   & = \mathbf{A}^{-1}-\mathbf{A}^{-1}\mathbf{B}\left(\mathbf{C}^{-1}+\mathbf{D}\mathbf{A}^{-1}\mathbf{B}\right)^{-1}\mathbf{D}\mathbf{A}^{-1},
\end{aligned}
\end{equation}
and get
\begin{equation}
\begin{aligned}
   \mathbf{\Omega} &= \frac{1}{\lambda}\mathbf{W}^{-1}\mathbf{X}\left(\mathbf{I}_{n}+\frac{1}{\lambda}\mathbf{X}^{T}\mathbf{W}^{-1}\mathbf{X}\right)^{-1}\mathbf{X}^{T}\left(\mathbf{I}_{d}+\frac{\eta}{2\lambda}\mathbf{W}^{-1}\right)\\
   &-\frac{\eta}{2\lambda}\mathbf{W}^{-1}.
\end{aligned}
\end{equation}
Define $\mathbf{N}=\frac{1}{\lambda}\mathbf{W}^{-1}\mathbf{X}\left(\mathbf{I}_{n}+\frac{1}{\lambda}\mathbf{X}^{T}\mathbf{W}^{-1}\mathbf{X}\right)^{-1}\mathbf{X}^{T}$. In order to further accelerate the computational speed, we transform the matrix multiplication $\mathbf{N}\left(\mathbf{I}_{d}+\frac{\eta}{2\lambda}\mathbf{W}^{-1}\right)$ into a Hadamard product $\mathbf{N}*\mathbf{P}$, where each row vector of $\mathbf{P}$ is a copy of $diag\left(\mathbf{I}_{d}+\frac{\eta}{2\lambda}\mathbf{W}^{-1}\right)$. It is worth noting that because $\mathbf{W}$ is a diagonal matrix, its inverse can be easily calculated by computing $1/(w_{ii}+\epsilon_{3})$, where $\epsilon_{3}$ is added in case $w_{ii}$ approximate zero. Finally, we obtain another version of (\ref{derivative}) to handle the situation where $d\gg{n}$:
\begin{equation}
   \mathbf{\Omega} = \mathbf{N}*\mathbf{P}-\frac{\eta}{2\lambda}\mathbf{W}^{-1}.
\end{equation}

\subsubsection{PSD projection}
The next step is to perform the PSD projection of $\boldsymbol{\Omega}$. Theoretically, $\mathbf{\Omega}$ obtained by (\ref{derivative}) is unlikely to be a symmetric matrix, which is often the case in practice. So we should first project it onto the symmetric cone. For an arbitrary square matrix $\mathbf{M}\in{\mathbb{R}^{d\times{d}}}$, we denote the PSD projection operator as $P_{S^d_+}(\mathbf{M})$. First, $\mathbf{M}$ can be projected onto the symmetric cone by calculating
\begin{equation}
    \mathbf{M}'=\Pi_{S^d}(\mathbf{M})=\frac{1}{2}(\mathbf{M}+\mathbf{M}^T).
\end{equation}\par
Then, for the symmetric matrix $\mathbf{M}'$, we project it onto the PSD cone by exploiting the proximal operator\cite{PSD} 
\begin{equation}
    \Pi_{S^d_+}(\mathbf{M'})=\sum\limits_{i=1}^{d}\left(\sigma_i\right)_+\mathbf{u}_i\mathbf{u}^T_i,
\end{equation}
where $\sum\nolimits_{i=1}^{d}\sigma_i\mathbf{u}_i\mathbf{u}^T_i$ is the eigenvalue decomposition of $\mathbf{M'}$. In other words, we conduct $\mathbf{M'}$'s eigenvalue expansion and drop the negative eigenvalues.

So now we can give the PSD projection of $\mathbf{\Omega}$ as
\begin{equation}
    P_{S^d_+}(\mathbf{\Omega})=\Pi_{S^d_+}\left(\Pi_{S^d}(\mathbf{\Omega})\right).
\end{equation}
By exploiting the PSD projection, the algorithm points the right convergence direction, which significantly speeds up the convergence. We will see this benefit in the experiments later.

\subsubsection{Proposed algorithm}
Based the above deduction, we have two variables ($\mathbf{W}$ and $\mathbf{\Omega}$) to update in each iteration. Hence, we formally introduce our algorithm as follows: 

\noindent\textbf{Fix} $\mathbf{W}$ \textbf{and update} $\mathbf{\Omega}$

Fix $\mathbf{W}$, we update $\mathbf{\Omega}$ by calculating
\begin{equation}\label{Update Omega}
    \mathbf{\Omega}= \left\{
    \begin{aligned}
        &P_{S^d_+}\left(\mathbf{N}*\mathbf{P}-\frac{\eta}{2\lambda}\mathbf{W}^{-1}\right)\quad \text{if}\; d\gg{n},\\
        &P_{S^d_+}\left(\left(\mathbf{S}-\frac{\eta}{2}\mathbf{I}_d\right)\left(\mathbf{S}+\lambda{\mathbf{W}}+\epsilon_2\mathbf{I}_d\right)^{-1}\right)\; \text{otherwise}. 
    \end{aligned} 
    \right.
\end{equation}
\noindent\textbf{Fix} $\mathbf{\Omega}$ \textbf{and update} $\mathbf{W}$

When $\mathbf{\Omega}$ is obtained, we can update $\mathbf{W}$ by calculating $\left(1/\left(2\sqrt{\boldsymbol{\omega}_j^T\boldsymbol{\omega}_j+\epsilon_1}\right)\right)(j=1,2,\cdots,d)$.

The proposed algorithm is summarized in \textbf{Algorithm 1}.
\begin{algorithm}[t] \label{Algorithm 1}
\caption{The optimization algorithm to solve Problem (\ref{OBJ})} 
\renewcommand{\algorithmicrequire}{\textbf{Input:}}
\renewcommand{\algorithmicensure}{\textbf{Output:}}
\begin{algorithmic}[1] \REQUIRE Data matrix $\mathbf{X}\in{\mathbb{R}^{d\times{n}}}$, regularization parameters $\lambda$ and $\eta$.
\ENSURE $h$ features of the data set
\STATE Initialize $\mathbf{S}=\mathbf{X}\mathbf{H}\mathbf{X}^T$, randomly initialize $\mathbf{\Omega}\in{S^{d}_+}$ and calculate $\mathbf{W}$ using the elements in $\mathbf{\Omega}$.
\STATE \textbf{repeat}
\STATE Update $\mathbf{\Omega}$ according to (\ref{Update Omega}).
\STATE Update $\mathbf{W}$ by calculating $\left(1/\left(2\sqrt{\boldsymbol{\omega}_j^T\boldsymbol{\omega}_j+\epsilon_1}\right)\right)(j=1,2,\cdots,d)$.
\STATE \textbf{until} converge
\STATE Sort $\Vert \boldsymbol{\omega}_j\Vert_2(j=1,2,\cdots,d)$ in descending order, then select the top $h$ features.
\end{algorithmic}
\label{Al1}
\end{algorithm}

\section{Discussion}
In this section, We first analyze the convergence and the computational complexity of \textbf{Algorithm 1}. Then we give an explanation and analysis on the relationship between the proposed SPCA-PSD and other SPCA-based UFS methods.

\subsection{Convergence Analysis}
Now we prove the convergence of \textbf{Algorithm 1}.

First, let us introduce a useful lemma in \cite{Lemma1}.
\newtheorem{lemma}{Lemma}
\begin{lemma} \label{lemma1}
For any nonzero vectors $\mathbf{a}$, $\mathbf{b}\in{\mathbb{R}^{c\times{1}}}$, the following inequality holds:
\begin{equation}
    {\Vert \mathbf{a}\Vert_2}-\frac{{\Vert\mathbf{a}\Vert_2^2}}{2{\Vert \mathbf{b}\Vert_2}} \leq {\Vert \mathbf{b}\Vert_2}-\frac{{\Vert \mathbf{b}\Vert_2^2}}{2{\Vert \mathbf{b}\Vert_2}}.
\end{equation}
\end{lemma}
Based on \textbf{Lemma} \ref{lemma1}, we propose the following:

\begin{theorem}\label{theorem 2}
With \textbf{Algorithm 1}, the updated $\mathbf{\Omega}$ can decrease the value of the objective function in Problem (\ref{OBJ}) until converge.
\end{theorem}
\begin{proof}
we denote the updated $\mathbf{\Omega}$ as $\mathbf{\Tilde{\Omega}}$, the updated $\mathbf{W}$ as $\mathbf{\Tilde{W}}$. For Problem (\ref{Problem 14}), it's easy to derive the following inequality
\begin{equation}\label{Ineq 1}
    \begin{aligned}
        Tr\left(\mathbf{\Tilde{\Omega}}\mathbf{S}\mathbf{\Tilde{\Omega}}^T\right)&-2Tr\left(\mathbf{S}\mathbf{\Tilde{\Omega}}\right)+\lambda Tr(\mathbf{\Tilde\Omega}\mathbf{\Tilde{W}}\mathbf{\Tilde\Omega}^T)+\eta Tr(\mathbf{\Tilde\Omega})\\
        &\leq Tr(\mathbf{\Omega}\mathbf{S}\mathbf{\Omega}^T)-2Tr(\mathbf{S}\mathbf{\Omega})+\lambda Tr(\mathbf{\Omega}\mathbf{W}\mathbf{\Omega}^T)\\
        &+\eta Tr(\mathbf{\Omega}).
    \end{aligned}
\end{equation}\par
For the simplicity of expression, we denote $J(\mathbf{\Omega})=Tr(\mathbf{\Omega}\mathbf{S}\mathbf{\Omega}^T)-2Tr(\mathbf{S}\mathbf{\Omega})+\lambda Tr(\mathbf{\Omega}\mathbf{W}\mathbf{\Omega}^T)+\eta Tr(\mathbf{\Omega})$ and $J(\mathbf{\Tilde\Omega})$ as the updated $J(\mathbf{\Omega})$. Then inequality (\ref{Ineq 1}) can be rewritten as 
\begin{equation}\label{Ineq 2}
    J(\mathbf{\Tilde\Omega}) + \lambda Tr(\mathbf{\Tilde\Omega}\mathbf{\Tilde{W}}\mathbf{\Tilde\Omega}^T) \leq J(\mathbf{\Omega}) + \lambda Tr(\mathbf{\Omega}\mathbf{W}\mathbf{\Omega}^T).
\end{equation}\par
If we add the same item $\sum\limits_{j=1}^d{\lambda\epsilon}/{2\sqrt{\boldsymbol{\omega}_j^T\boldsymbol{\omega}_j+\epsilon}}$ to both sides of (\ref{Ineq 2}), we will get
\begin{equation}\label{Ineq 3}
\begin{aligned}
    J(\mathbf{\Tilde\Omega}) &+ \lambda Tr(\mathbf{\Tilde\Omega}\mathbf{\Tilde{W}}\mathbf{\Tilde\Omega}^T) + \sum\limits_{j=1}^d\frac{{\lambda\epsilon}}{{2\sqrt{\boldsymbol{\omega}_j^T\boldsymbol{\omega}_j+\epsilon}}}  \\&\leq J(\mathbf{\Omega}) + \lambda Tr(\mathbf{\Omega}\mathbf{W}\mathbf{\Omega}^T)+\sum\limits_{j=1}^d\frac{{\lambda\epsilon}}{{2\sqrt{\boldsymbol{\omega}_j^T\boldsymbol{\omega}_j+\epsilon}}}.
\end{aligned}
\end{equation}

Since $\mathbf{\Omega}$ is a symmetric matrix, we have
\begin{equation}\label{equ 1}
    Tr(\mathbf{\Omega}\mathbf{W}\mathbf{\Omega}^T)=Tr(\mathbf{\Omega}^T\mathbf{W}\mathbf{\Omega})=\sum\limits_{j=1}^d\frac{{\lambda{\boldsymbol{\omega}_j^T\boldsymbol{\omega}_j}}}{{2\sqrt{\boldsymbol{\omega}_j^T\boldsymbol{\omega}_j+\epsilon}}}.
\end{equation}

The same equality holds for $Tr(\mathbf{\Tilde\Omega}\mathbf{W}\mathbf{\Tilde\Omega}^T)$. Thus, we substituted (\ref{equ 1}) into (\ref{Ineq 3}) and get
\begin{equation} \label{Ineq 4}
    J(\mathbf{\Tilde\Omega})+\sum\limits_{j=1}^d\frac{{\lambda({\boldsymbol{\Tilde\omega}_j^T\boldsymbol{\Tilde\omega}_j}+\epsilon)}}{{2\sqrt{\boldsymbol{\Tilde\omega}_j^T\boldsymbol{\Tilde\omega}_j+\epsilon}}} \leq J(\mathbf{\Omega})+\sum\limits_{j=1}^d\frac{{\lambda({\boldsymbol{\omega}_j^T\boldsymbol{\omega}_j}+\epsilon)}}{{2\sqrt{\boldsymbol{\omega}_j^T\boldsymbol{\omega}_j+\epsilon}}}.
\end{equation}

According to \textbf{Lemma} \ref{lemma1}, we have
\begin{equation}\label{Ineq 5}
    \sqrt{{\boldsymbol{\Tilde\omega}_j^T\boldsymbol{\Tilde\omega}_j}+\epsilon}-\frac{{{\boldsymbol{\Tilde\omega}_j^T\boldsymbol{\Tilde\omega}_j}+\epsilon}}{{2\sqrt{\boldsymbol{\Tilde\omega}_j^T\boldsymbol{\Tilde\omega}_j+\epsilon}}} \leq \sqrt{{\boldsymbol{\omega}_j^T\boldsymbol{\omega}_j}+\epsilon}-\frac{{{\boldsymbol{\omega}_j^T\boldsymbol{\omega}_j}+\epsilon}}{{2\sqrt{\boldsymbol{\omega}_j^T\boldsymbol{\omega}_j+\epsilon}}}.
\end{equation}

Further, we can accumulate (\ref{Ineq 4}) from $j=1$ to $j=d$. Considering $\lambda>0$, we can get
\begin{equation}\label{Ineq 6}
\begin{aligned}
        \lambda\sum\limits_{j=1}^d\sqrt{{\boldsymbol{\Tilde\omega}_j^T\boldsymbol{\Tilde\omega}_j}+\epsilon}-\lambda\sum\limits_{j=1}^d\frac{{{\boldsymbol{\Tilde\omega}_j^T\boldsymbol{\Tilde\omega}_j}+\epsilon}}{{2\sqrt{\boldsymbol{\Tilde\omega}_j^T\boldsymbol{\Tilde\omega}_j+\epsilon}}} \\\leq \lambda\sum\limits_{j=1}^d\sqrt{{\boldsymbol{\omega}_j^T\boldsymbol{\omega}_j}+\epsilon}-\lambda\sum\limits_{j=1}^d\frac{{{\boldsymbol{\omega}_j^T\boldsymbol{\omega}_j}+\epsilon}}{{2\sqrt{\boldsymbol{\omega}_j^T\boldsymbol{\omega}_j+\epsilon}}}.
\end{aligned}
\end{equation}

By summing (\ref{Ineq 4}) and (\ref{Ineq 6}), we have
\begin{equation}
    J(\mathbf{\Tilde\Omega})+ \lambda\sum\limits_{j=1}^d\sqrt{{\boldsymbol{\Tilde\omega}_j^T\boldsymbol{\Tilde\omega}_j}+\epsilon} \leq J(\mathbf{\Omega}) +  \lambda\sum\limits_{j=1}^d\sqrt{{\boldsymbol{\omega}_j^T\boldsymbol{\omega}_j}+\epsilon}.
\end{equation}\par
Finally, we can get the following inequality
\begin{equation}
    J(\mathbf{\Tilde\Omega})+\lambda{\Vert {\mathbf{\Tilde\Omega}}\Vert_{2,1}} \leq J(\mathbf{\Omega})+\lambda{\Vert {\mathbf{\Omega}}\Vert_{2,1}}.
\end{equation}
\end{proof}

Now we have proved that $\mathbf{\Tilde{\Omega}}$ can decrease the value of the object function in Problem (\ref{OBJ}) in each iteration. Since all the terms in the objective function obviously have a lower bound $0$, the convergence is guaranteed.
\subsection{Computational Complexity Analysis}
The computational complexity of \textbf{Algorithm 1} can be decomposed into the following parts:
\begin{itemize}[leftmargin=*]
\item We need $O\left(d^2n\right)$ to initialize $\mathbf{S}$, $O\left(d^2\right)$ to initialize $\mathbf{\Omega}$, so the total initialization takes $O\left(d^2n\right)$.
\item In each iteration, it takes at most $O\left(d^3\right)$ to update $\mathbf{\Omega}$ due to the calculation of matrix inversion and the PSD projection (whether $d\gg{n}$ or not). It takes $O\left(d^2\right)$ to update $\mathbf{W}$. So the total computational complexity of one iteration is at most $O\left(d^3\right)$.
\item It takes $O\left(d^2\right)$ to calculate $\Vert \boldsymbol{\omega}_j\Vert_2(j=1,2,\cdots,d)$ and $O\left(d\log_{2}d\right)$ to sort them.
\end{itemize}

In conclusion, the total computational complexity of \textbf{Algorithm 1} is $O\left(d^2n+d^3t\right)$, where $t$ is the number of iterations. In practice, \textbf{Algorithm 1} converges fast such that $t$ is usually under 50. We compare the computational complexity of our proposed SPCA-PSD with other UFS methods used in the following comparison experiments. The result is presented in TABLE \ref{Tab 1}, where $m$ is the reduced dimension, $c$ denotes the number of classes in data, $t_i$ refers to the number of the $i$-th subproblem's iterations. We can see that only the computational complexity of SPCA-PSD and SPCAFS in one iteration isn't relevant to $n$, which means it will not expand rapidly as the number of samples increases. Although SPCAFS and SPCA-PSD have similar computational complexity because they both rely on eigenvalue decomposition, SPCAFS often takes more iterations to converge due to its non-complexity. Therefore in practice, SPCA-PSD usually costs less time than SPCAFS, which will be seen in the running time experiment.

\begin{table}[!t]
\caption{
Computational complexity of comparative UFS methods. $d$, $n$ and $t$ refer to the number of features,   the number of samples and the number of iterations, respectively. $t_i$ refers to the number of the $i$-th subproblem's iterations. }\label{Tab 1}
\centering

\setlength\tabcolsep{3pt}
\begin{tabular}{c c c c}
\specialrule{0em}{5pt}{5pt}
\toprule
  \textbf{Methods}   &  \textbf{Computational Complexity} & \textbf{Type}    & \textbf{Convexity}\\
 \midrule
   LapScore\cite{LS}  & $O\left(dn^2\right)$                       & spectral   &  No optimization\\
   \specialrule{0em}{1pt}{1pt}
   UDFS\cite{UDFS}      & $O\left(dn^2+td^3\right)$                  & spectral   &  non-convex\\
   \specialrule{0em}{1pt}{1pt}
   SOGFS\cite{SOGFS}     & $O\left(dn^2+t(d^3t_1+n^2m+ndm)\right)$    & spectral   &  non-convex\\
   \specialrule{0em}{1pt}{1pt}
   RNE\cite{RNE}       & $O\left(dn^2+t(t_1(n+s)d^2+t_2n))\right)$  & spectral        &  convex\\
   \specialrule{0em}{1pt}{1pt}
   SPCAFS\cite{SPCAFS}    & $O\left(d^2n+td^3\right)$                  & SPCA      &  non-convex\\ 
   \specialrule{0em}{1pt}{1pt}
   AW-SPCA\cite{AW-SPCA}   & $O\left(t(d^3+dn^2)\right)$                & SPCA      &  convex\\   
   \specialrule{0em}{1pt}{1pt}
   CSPCA\cite{CSPCA}     & $O\left(t(d^3+n^2)\right)$                       & SPCA       &  convex\\ 
   \specialrule{0em}{1pt}{1pt}
   SPCA-PSD  & $O\left(d^2n+td^3\right)$                       & SPCA       &  convex\\ 
 \bottomrule 
\end{tabular}
\end{table}

\subsection{Relationship between SPCA-PSD and other SPCA-based UFS Methods}
\subsubsection{Adaptive weight matrix of principal components: the advantage of SPCA-PSD over SPCAFS}
As we mention in Section 2, SPCAFS\cite{SPCAFS} constructs its model by directly incorporate a sparsity constraint into the optimization problem of PCA. In fact, the model of SPCAFS is equivalent to 
\begin{equation}
\begin{aligned} \label{Problem 6}
    \min\limits_{\mathbf{U}} \quad &{\Vert {\mathbf{X}-\mathbf{U}\mathbf{U}^T\mathbf{X}} \Vert}^2_F+\lambda{\Vert \mathbf{U}^T \Vert}_{2,p}\\
    s.t. \quad&{\mathbf{U}^T\mathbf{U}=\mathbf{I}_k}.
\end{aligned} 
\end{equation}
Apparently it is a non-convex model that cannot guarantee to converge at the global optimal solution theoretically. And the meaning of a convex formulation rooted in the original SPCA is more than just a better convergence. It can be observed that SPCAFS also uses reconstruction error as criterion. However, compared with SPCA-PSD, the reconstruction matrix counterpart in SPCAFS is calculated by $\mathbf{U}\mathbf{U}^T$. Now we take a look at the reconstruction matrix $\mathbf{\Omega}$ in SPCA-PSD. Given the eigenvalue decomposition of $\mathbf{\Omega}=\mathbf{U}\mathbf{D}\mathbf{U}^T$, it is reasonable to viewed the diagonal elements of $\mathbf{D}$ as adaptive weights corresponding to each vector in the orthogonal matrix $\mathbf{U}$. Therefore, during the learning process, $\mathbf{D}$ adaptively adds importance on certain components in the low-dimensional subspace that is projected from the original feature space, guided by the reconstruction error. In other words, $\mathbf{D}$ selects the most necessary ones of the principal components. Hence, a convex original-SPCA-based model is more likely to perform the data reconstruction better and thus have a better ability at judging the significant of a feature. Later experimental result shows that SPCA-PSD outperforms SPCAFS in most cases, which will prove the above analysis.      
\subsubsection{What does the PSD constraint mean to a convex SPCA-based UFS model}\label{Difference}
In \textbf{Appendix A}, we prove that for existing convex SPCA-based models for UFS: CSPCA \cite{CSPCA} and AW-SPCA\cite{AW-SPCA}, the PSD constraint still holds. Now we analyze the difference between these two models and our proposed SPCA-PSD:
\begin{itemize}[leftmargin=*]
    \item \textbf{AW-SPCA}. The optimization problem of AW-SPCA is as (\ref{AW-SPCA}), where $\mathbf{v}=(\mathbf{I}-\mathbf{\Omega})\mathbf{b}$ and $\mathbf{b}$ is the optimal $\ell_{2,1}$ distance based mean of data. AW-SPCA utilizes $\ell_{2,1}$-norm to form the reconstruction error term. And it focuses on finding $\mathbf{b}$, in order to centralize the data properly and enhance the robustness to outliers in the data. However, AW-SPCA doesn't consider the low-rank constraint of the reconstruction matrix that gives a PCA-based model the global manifold learning ability. And it also ignores the PSD constraint of $\mathbf{\Omega}$. 
    \begin{equation}\label{AW-SPCA}
        \min\limits_{\mathbf{\Omega},\mathbf{v}}\quad {\Vert {\mathbf{X}-\mathbf{\Omega}\mathbf{X}-\mathbf{v}\mathbf{1}^T} \Vert}_{2,1}+\lambda{\Vert \mathbf{\Omega} \Vert}_{2,1}
    \end{equation}
    \item \textbf{CSPCA}. The optimization problem of CSPCA is as (\ref{CSPCA}). CSPCA also uses $\ell_{2,1}$-norm to describe the reconstruction error. And it incorporates the trace norm (i.e. $Tr\left(\left(\mathbf{\Omega}\mathbf{\Omega}^T\right)^\frac{1}{2}\right)$) as the convex approximation to the low-rank constraint. In each iteration, CSPCA need to calculate the trace norm in order to obtain the value of the objective value, which adds computational burden on optimization. SPCA-PSD, on the other hand, is able to simplify the nuclear norm (which is also a common used convex approximation to the low-rank constraint) to the trace function of the reconstruction matrix, thanks to the PSD constraint. Obviously, the trace function is much easier to calculate and optimize.   
        \begin{equation}\label{CSPCA}
        \min\limits_{\mathbf{\Omega}}\quad {\Vert {\mathbf{X}-\mathbf{\Omega}\mathbf{X}} \Vert}_{2,1}+\lambda{\Vert \mathbf{\Omega} \Vert}_{2,1} + \eta{Tr\left(\left(\mathbf{\Omega}\mathbf{\Omega}^T\right)^\frac{1}{2}\right)}
    \end{equation}
\end{itemize}

Both CSPCA and AW-SPCA can be viewed as a robust version (of our standard convex SPCA-based model) that utilizes $\ell_{2,1}$ to describe the reconstruction error. However, due to the lack of PSD constraint, CSPCA and AW-SPCA both design an optimization algorithm purely relying on the derivative of the objective function, resulting in more iterations to search the optimal solution. According to the experimental results provided by the authors of AW-SPCA\cite{AW-SPCA}, it takes hundreds of iterations for AW-SPCA to converge on some datasets. And as we will see later, CSPCA comes out with a similar drawback according to our experimental results.

\subsection{AW-SPCA-PSD and CSPCA-PSD: Variations of SPCA-PSD}
Based on the analysis in Appendix A and Section \ref{Difference}, we further add the PSD constraint to the models of AW-SPCA and CSPCA respectively, and obtain the optimization problems of their PSD versions as (\ref{AW-SPCA-PSD}) and (\ref{CSPCA-PSD}). We optimize them with similar steps used in SPCA-PSD: first calculate the derivative of the objective function, then conduct the PSD projection. 
\begin{equation}\label{AW-SPCA-PSD}
\begin{aligned}
    \min\limits_{\mathbf{\Omega},\mathbf{v}} \quad &{\Vert {\mathbf{X}-\mathbf{\Omega}\mathbf{X}-\mathbf{v}\mathbf{1}^T}\Vert}_{2,1}+\lambda{\Vert \mathbf{\Omega} \Vert}_{2,1},\\
    s.t. \quad&{\mathbf{\Omega}\in{S^d_+}}.
\end{aligned}
\end{equation}
\par It is worth noting that with the PSD constraint, CSPCA-PSD is able to use the trace function as a convex approximation to the low-rank constraint, which releases the burden of calculating the trace norm in each iteration. In fact, CSPCA-PSD can be viewed as a variant of SPCA-PSD, considering the robustness to outliers in data. 
\begin{equation}\label{CSPCA-PSD}
\begin{aligned}
        \min\limits_{\mathbf{\Omega}}\quad &{\Vert {\mathbf{X}-\mathbf{\Omega}\mathbf{X}} \Vert}_{2,1}+\lambda{\Vert \mathbf{\Omega} \Vert}_{2,1} + \eta{Tr(\mathbf{\Omega})},\\
    s.t. \quad&{\mathbf{\Omega}\in{S^d_+}}.
\end{aligned}
\end{equation}
\par In this paper, we regard AW-SPCA-PSD and CSPCA-PSD as the variations of our proposed SPCA-PSD in the following experiments.

\section{Experiments}
In this section, we conduct six experiments to demonstrate the effectiveness of our proposed SPCA-PSD, CSPCA-PSD and AW-SPCA-PSD and provide corresponding analysis.  
\subsection{Experimental Settings}
\subsubsection{Experimental environment and datasets}
All the  experiments are conducted in MATLAB R2020a on a personal computer with 3.7-GHz R9-5900X CPU and 128GB main memory under the environment of Windows 10 operation system. The experiments are carried out on three synthetic datasets and twelve real-world datasets. The synthetic datasets include Two-moon data, Three-ring data and Three-curve data. The real-world datasets include five image datasets (PIE\cite{PIE}, Imm40\cite{Imm40}, Orlraw10P\footnote{https://jundongl.github.io/scikit-feature/datasets.html}, a subset of MNIST\footnote{http://yann.lecun.com/exdb/mnist/} and USPS\cite{USPS}), 
one spoken letter dataset (Isolet\cite{Isolet}), two biological datasets (Lung\cite{lung} and ALLAML\cite{ALLAML}) and two deep learning datasets( Indoor\underline{ }Resnet50\cite{SPCAFS} and MSTAR\underline{ }SOC\underline{ }CNN). The detail of all datasets is shown in Table \ref{Dataset}. For image datasets, we normalize the value of each feature to the range of 0 and 1, in order to narrow the grid search range of regularization parameters. For MSTAR\underline{ }SOC\underline{ }CNN, a pre-trained convolutional neural network is used to extract deep features from the standard operation condition (SOC) test set in MSTAR\cite{MSTAR}, where we collect the output of the first full connection layer as descriptors. 

\begin{table}[!t]
\caption{Statistics of datasets.} \label{Dataset}
\centering
\begin{tabular}{c|cccc}
\specialrule{0em}{1pt}{5pt}
\toprule
  \makebox[0.03\textwidth][c]{\textbf{Type}}
  &\makebox[0.05\textwidth][c]{\textbf{Dataset}}
  & \makebox[0.055\textwidth][c]{\textbf{\# Features}}
  & \makebox[0.055\textwidth][c]{\textbf{\# Samples}}   
  &\makebox[0.055\textwidth][c]{\textbf{\# Classes}}
  \\
 \midrule
   \multirow{3}{*}{Synthetic} 
   &Two-moon             &9 &800 &2\\
   \specialrule{0em}{1pt}{1pt}
   &Three-curve          &9 &900 &3\\
   \specialrule{0em}{1pt}{1pt}
   &Three-ring             &9 &900 &3\\ 
   \specialrule{0em}{1pt}{1pt} \hline \specialrule{0em}{1pt}{1pt}
   \multirow{10}{*}{Real-world}
   &USPS\cite{USPS}             &256 &9298 &10\\
   \specialrule{0em}{1pt}{1pt}
   &Isolet\cite{Isolet}           &617 &1560 &26\\
   \specialrule{0em}{1pt}{1pt}
   &MNIST            &784 &20000 &10\\
   \specialrule{0em}{1pt}{1pt}
   &Imm40\cite{Imm40}            &1024&240  &40 \\
   \specialrule{0em}{1pt}{1pt}
   &PIE\cite{PIE}              &1024&1166 &53 \\
   \specialrule{0em}{1pt}{1pt}
   &MSTAR\underline{ }SOC\underline{ }CNN
                    &1024&2425 &10 \\
   \specialrule{0em}{1pt}{1pt}
   &Indoor\underline{ }Resnet50\cite{SPCAFS}   &1024&15620&67 \\
   \specialrule{0em}{1pt}{1pt}
   &Lung\cite{lung}             &3312&203  &5  \\
   \specialrule{0em}{1pt}{1pt}
   &ALLAML\cite{ALLAML}           &7129&72   &2  \\
   \specialrule{0em}{1pt}{1pt}
   &Orlraws10P      &10304&100   &10  \\

 \bottomrule
\end{tabular}
\end{table}

\subsubsection{Comparative methods}
To evaluate the effectiveness of our proposed SPCA-PSD, CSPCA-PSD and AW-SPCA-PSD, we compare them with several state-of-the-art UFS methods of different types: LapScore\footnote{https://github-com-s.libyc.nudt.edu.cn/ZJULearning/MatlabFunc/}\cite{LS}, UDFS\footnote{http://www.cs.cmu.edu.libyc.nudt.edu.cn/~yiyang/}\cite{UDFS}, 
SOGFS\footnote{http://www.escience.cn/people/fpnie/index.html}\cite{SOGFS}, RNE\footnote{ https://github.com/liuyanfang023/KBS-RNE.}{}\cite{RNE} , SPCAFS\footnote{https://github.com/quiter2005/algorithm}\cite{SPCAFS}, CSPCA\cite{CSPCA} and AW-SPCA\cite{AW-SPCA}. And we use \emph{All Features} as a baseline method, which uses all features to perforrm clustering. Since we already have a detailed description of SPCA-based methods in Section 4, we only introduce the spectral-based methods as follows:
\begin{itemize}[leftmargin=*]
    \item \textbf{LapScore}\cite{LS}: LapScore is a filter method that constructs a knn similarity matrix and calculates predefined Laplacian score for each feature. The Laplacian score aims to describe the ability of preserving the local geometric structure of data.
    \item \textbf{UDFS}\cite{UDFS}: UDFS constructs a knn local set and defines a linear classifier matrix that projects each sample to a scaled cluster label. By utilizing $\ell_{2,1}$-norm on the projection matrix, UDFS aims to preserve local discriminative information of data.
    \item \textbf{SOGFS}\cite{SOGFS}: SOGFS optimizes an adaptive similarity matrix to obtain the ideal probabilistic neighborhood which contains exact $c$ (the number of classes) components, and conducts local manifold learning by minimizing the distance between each two nearest samples after they are projected into a low dimensional space by the feature selection matrix.
    \item \textbf{RNE}\cite{RNE}: RNE makes use of the locally linear embedding algorithm to obtain an optimal projection matrix that keeps a sample and its neighbors closed to each other after being projected onto a local manifold. And it uses $\ell_1$-norm to describe the reconstruction error.
\end{itemize}

Codes for AW-SPCA and CSPCA are implemented by ourselves. Codes of our proposed SPCA-PSD, AW-SPCA-PSD and CSPCA-PSD are publicly available\footnote{https://github.com/zjj20212035/SPCA-PSD}. Codes for other methods are provided by their original authors as footnoted.

\subsubsection{Parameter settings}
For all spectral-based methods, we set the number of neighbors $k=5$ and Gaussian kernel width $\sigma=1$. For LapScore, we use binary connection to construct similarity matrix. For SOGFS and SPCAFS, we set the reduced dimension as 
$m=c-1$ ($c$ is the number of classes in a dataset) and set $p=0.5$ for the $\ell_{2,p}$-norm. For regularization parameters, we tune them by grid search strategy from $\left\{10^{-4},10^{-3},10^{-2},10^{-1},1,10,10^{2},10^{3},10^{4}\right\}$ for the sake of fairness, and report the best results of all methods. For datasets that contain more than 300 features, we set the number of selected features as $\{50,100,150,200,250,300\}$, otherwise we set it as $\{10,30,50,70,90,110\}$. 

\subsubsection{Evaluation methodology and metrics}
In all experiments, after a UFS method has finished selecting features, we evaluate its performance by conducting K-means clustering and mapping the obtained pseudo labels to the real labels by adopting the Kuhn-Munkres algorithm. We use two classic metrics \emph{Accuracy (ACC)} and \emph{Normalized Mutual Information (NMI)}\cite{NMI} to describe the clustering results, and use \emph{Running time} to describe the efficiency of each algorithm. Both ACC and NMI are common metrics used in the field of UFS. Due to the K-means’ dependence on initialization, we repeat the clustering by 30 times and record the average values of ACC and NMI.
\begin{itemize}[leftmargin=*]
\item \textbf{ACC}: ACC is used to describe the accuracy of clustering, and is calculated according to (\ref{ACC}) and (\ref{indicator}), where $n$ is the total number of samples, $map(i)$ denotes the clustering label of the $i$-th sample after mapping to the real label, and $label(i)$ represents the ground truth label. Larger ACC indicated better performance.
\begin{equation}\label{ACC}
    \text{ACC}=\frac{1}{n}\sum\limits_{i=1}^{n}\delta\left(map(i),label(i)\right).
\end{equation}
\begin{equation}\label{indicator}
    \delta(p,q)=\left\{
    \begin{aligned}
        &1\quad \text{if}\; p=q,\\
        &0\quad \text{otherwise}. 
    \end{aligned} 
    \right.
\end{equation}
\item \textbf{NMI}: NMI is used to describe the mutual dependence between clustering results (after being mapped into the real labels) and ground truth labels. Given the clustering results $\mathbf{m}$ and the ground truth labels $\mathbf{l}$ (both include all samples), $I(\mathbf{m},\mathbf{l})$ is the mutual information between $\mathbf{m}$ and $\mathbf{l}$. $H(\mathbf{m})$ and $H(\mathbf{l})$ are the entropy of $\mathbf{m}$ and $\mathbf{l}$ respectively. The same as ACC, Larger NMI hints better clustering performance.
\begin{equation}
    \text{NMI}(\mathbf{m},\mathbf{l}) = \frac{I(\mathbf{m},\mathbf{l})}{\sqrt{H(\mathbf{m})H(\mathbf{l})}}.
\end{equation}
\item \textbf{Running time}. For an algorithm, the running time starts from the initialization and ends when the algorithm outputs the selected features. 
\end{itemize}

\subsection{Experiment 1: Evaluation of CSPCA-PSD and AW-SPCA-PSD}
In this section, we perform an experiment on our proposed AW-SPCA-PSD and CSPCA-PSD to support our deductions in Appendix A, where we prove that the true solution space of AW-SPCA and CSPCA is also the PSD cone. We run CSPCA, AW-SPCA, CSPCA-PSD and AW-SPCA-PSD on three real-world datasets: Isolet\cite{Isolet}, USPS\cite{USPS} and PIE\cite{PIE}. For fair comparison, we set the values of all regularization parameters to be 10 (which is always corresponding to the best clustering performance and means more iterations for AW-SPCA and CSPCA), and the number of selected features to be 100. The stop criterion is defined as having the absolute difference between two objective function values of adjacent iterations reach $10^{-5}$. The comprehensive performance results (including ACC, NMI, number of iterarions and running time) are shown in Table \ref{PSD test}. The convergence curves on PIE are shown in Fig.\ref{PSD Convergence}. Upon observation, we can see that with the PSD projection, CSPCA-PSD and AW-SPCA-PSD can achieve similar or even better clustering performance than CSPCA and AW-SPCA on the whole, while taking less iterations and running time to converge. And from Fig. we can see that the PSD projection accelerates the convergence. we should point it out that the PSD projection indeed brings extra computational complexity, which explains why the ratio of number of iterations to running time isn't the same between AW-SPCA and AW-SPCA-PSD. Fortunately, in practice, when the regularization parameter is set to the value corresponding to the optimal feature subset, the PSD version often takes far less iterations and running time than the original edition. Another point worth noticing is that there is no evidence that AW-SPCA-PSD and CSPCA-PSD definitely outperform their original editions in every applications. The PSD constraint is necessary only from the perspective of SPCA. But considering the comprehensive performance, we use the PSD versions in the following experiments to achieve both effectiveness and efficiency.  

\begin{table*}[!t] 
\caption{The clustering results (ACC\%$\pm$STD, NMI\%$\pm$STD, number of iterations,  running time (seconds)) of CSPCA, AW-SPCA, CSPCA-PSD and AW-SPCA-PSD on three real-world datasets. All the algorithms are run with each regularization parameter fixed to 10. We record the number of iterations and the running time when an algorithm reaches the given convergence threshold. From the table we can see that the PSD versions of CSPCA and AW-SPCA can achieve approximate or even better clustering ACC and NMI than their origin editions on the chosen data sets, while taking less iterations and running time.}
\label{PSD test}
\centering
\renewcommand\arraystretch{1.5}
\begin{tabular}{|c|c|c|c|}  \hline
  \diagbox{Method}{Dataset}                         & USPS                 & Isolet   & PIE                     \\ \hline
  CSPCA      & (44.80$\pm$2.14, 36.03$\pm$1.33, 192,  60.4239) & (43.51$\pm$2.34, 59.64$\pm$1.51, 201,  15.6377) &(43.81$\pm$2.85, 68.19$\pm$1.35, 204,  34.6264)\\ \hline
  CSPCA-PSD  & (44.75$\pm$1.74, 37.76$\pm$1.16, 32,  10.0499)& (45.17$\pm$2.49, 62.36$\pm$1.27, 25,  1.6912)  &(43.21$\pm$1.92, 67.99$\pm$1.14, 20,  2.6545)\\   \hline
  AW-SPCA    & (38.72$\pm$2.77, 32.88$\pm$1.02, 56,  31.0669)& (40.88$\pm$2.68, 56.77$\pm$1.23, 142, 11.4443) &(40.05$\pm$2.24, 64.84$\pm$1.31, 181,  21.5730)\\     \hline
  AW-SPCA-PSD &(39.58$\pm$3.10, 33.02$\pm$1.09, 38,  21.2972) & (41.25$\pm$2.92, 56.95$\pm$1.78, 46, 4.7831) &(42.06$\pm$2.24, 66.61$\pm$1.36, 35,  6.6160)\\     \hline
\end{tabular}
\end{table*}

\begin{figure}
    \centering
    \subfloat[]{\includegraphics[width=1.5in]{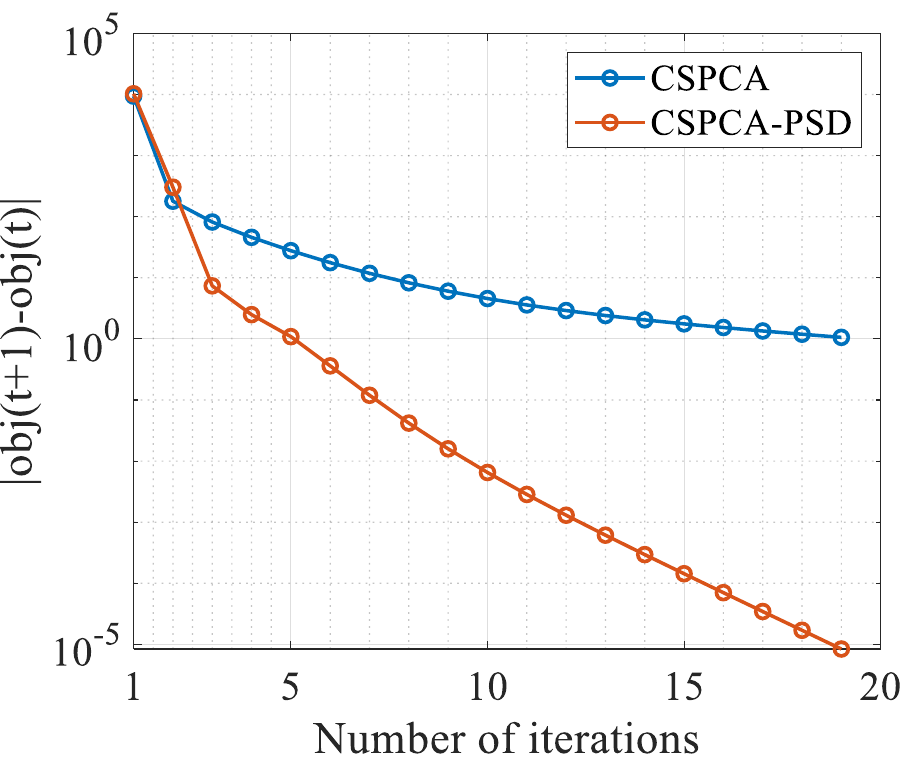}}\hspace{1em}
    \subfloat[]{\includegraphics[width=1.5in]{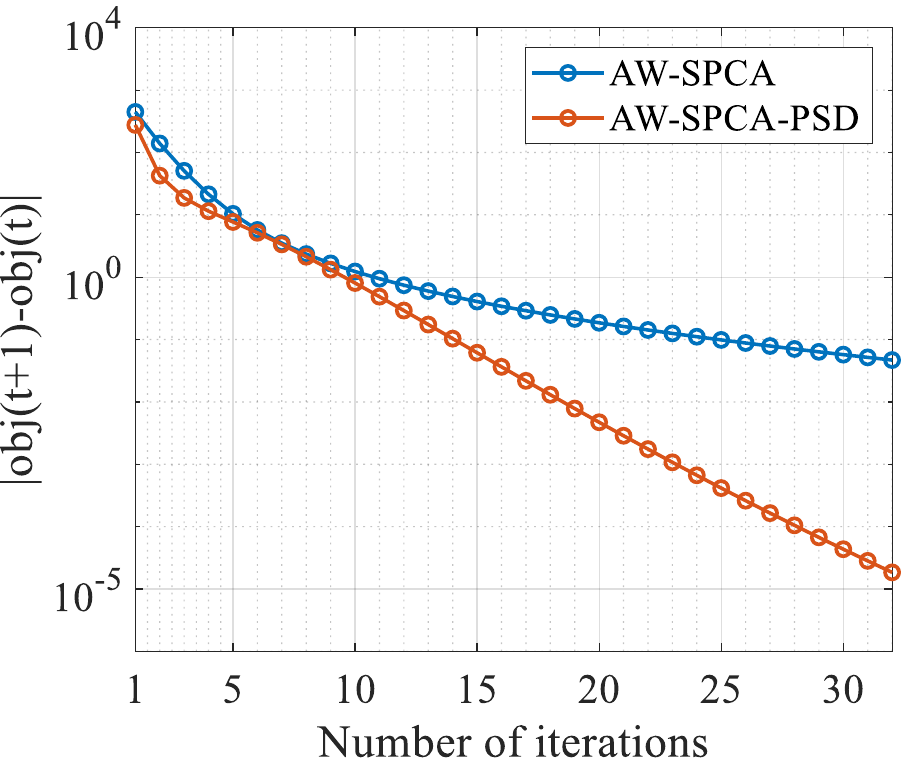}}
    \caption{Convergence curves on PIE. (a) Convergence curves of CSPCA and CSPCA-PSD. (b) Convergence curves of AW-SPCA and AW-SPCA-PSD. 'obj(t)' denotes the objective function value of the $t$-th iteration.}
    \label{PSD Convergence}
\end{figure}

\subsection{Experiment 2: Synthetic Datasets Experiments}
In this section, We run different types of UFS methods on three synthetic data sets: Two-moon, Three-ring and Three-curve. We generate these datasets by setting the first two features to obey distributions with certain shape while the rest seven features are Gaussian noise varying in amplitude. For each UFS method, after obtaining scores for all nine features, we select the top two features and perform clustering. After grid search, the feature subset corresponding to the best ACC are recorded. Then we present a scatter diagram of all samples with the selected two features as coordinates. The synthetic datasets and the feature selection results are shown in Fig.\ref{Experimental results on synthetic data}. We show the results of UDFS, SOGFS, SPCAFS and SPCA-PSD. They represent the predefined-graph based, the adaptive graph based, the non-convex SPCA-based and the convex SPCA-based methods, respectively. It can be seen that SPCA-PSD selects discriminative features on all three data sets, while other methods don’t always select the right features completely. For spectral-based methods UDFS and SOGFS, when faced with highly corrupted data, the similarity matrix becomes unreliable and misguides them to select noise features. For SPCA-based methods SPCAFS and SPCA-PSD, the global manifold learning ability makes them less sensitive to noise, therefore they can find the manifold hidden in noise interruption correctly. Compared with SPCAFS, SPCA-PSD learns an adaptive weight matrix that measures the importance of each component in the low-dimensional space, thus can better reconstruct the data and get rid of noise features.

\begin{figure*}[!t]
\begin{minipage}[b]{.5\linewidth}
    \subfloat{
    \includegraphics[width=7.5in]{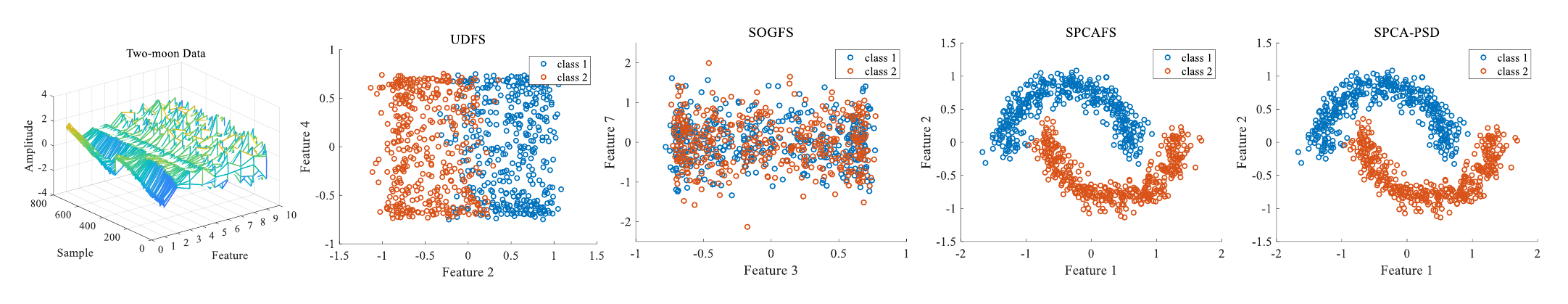}}
\end{minipage} \par
\begin{minipage}[b]{.5\linewidth}
    \subfloat{
    \includegraphics[width=7.5in]{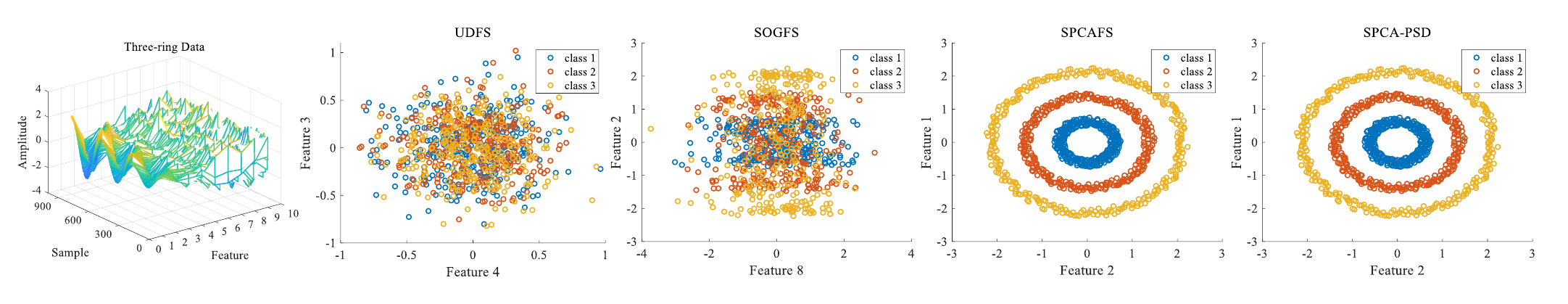}}
\end{minipage} \par
\begin{minipage}[b]{.5\linewidth}
    \subfloat{
    \includegraphics[width=7.5in]{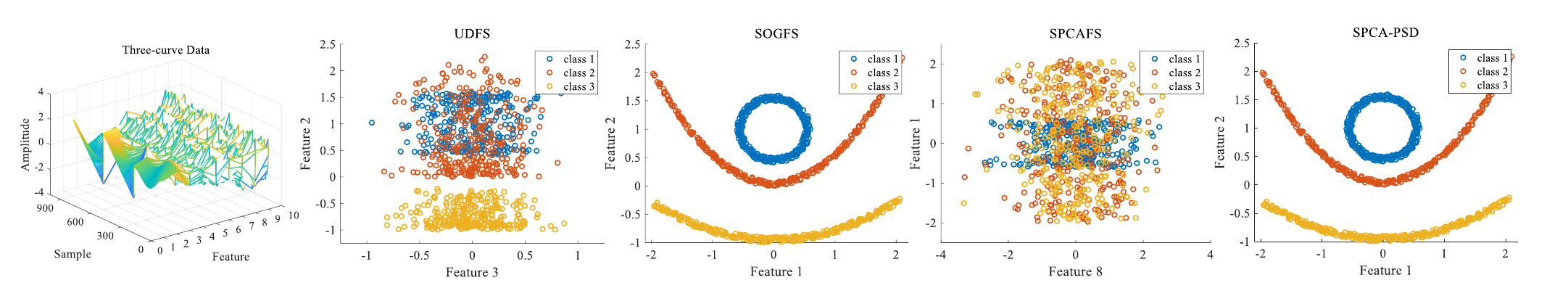}}
\end{minipage} \par
\caption{{Feature selection results on three synthetic data sets. The first image in each row is the corresponding synthetic data set on which the experimental results in the same row run. The first two features in each synthetic data set have a regular distribution in a shape of two-moon, three-ring or three-curve. For each method, we show the the first two features they selected. If one method selects both of the correct features, the scatter diagram will present a shape that distinguishes different classes completely, e.g., the result of SPCA-PSD on all three data sets.}}
\label{Experimental results on synthetic data}

\end{figure*}

\subsection{Experiment 3: Clustering Experiments on Real-world Datasets}
In this section, we conduct clustering experiments on ten real-world datasets: PIE\cite{PIE}, Imm40\cite{Imm40}, USPS\cite{USPS}, Lung\cite{lung}, Orlraws10P, Isolet\cite{Isolet}, MNIST, ALLAML\cite{ALLAML}, Indoor\underline{ }Resnet50\cite{SPCAFS}, MSTAR\underline{ }SOC\underline{ }CNN. And we choose PIE to visualize the selected features. 

\subsubsection{Clustering results of feature selection}
We run LapScore, UDFS, RNE, SOGFS, SPCAFS, AW-SPCA-PSD, CSPCA-PSD and SPCA-PSD on ten real-world datasets for clustering. For each method, we execute corresponding algorithm with regularization parameters from the given grid, and record the best ACC and NMI. The obtained ACC curves and NMI curves are shown in Fig.\ref{CFS_ACC} and Fig.\ref{CFS_NMI}.  Comparing the results of SPCA-PSD, AW-SPCA-PSD, CSPCA-PSD with other methods, we have the following observations:
\begin{itemize}[leftmargin=*]
    \item \textbf{SPCA-based methods generally outperform spectral-based methods.} Whether in terms of ACC or NMI, SPCAFS, SPCA-PSD, AW-SPCA-PSD and CSPCA-PSD get better performance on ten datasets. And spectral-based methods' performance varies more across different applications. Take LapScore for an example, its performance is on the same level of SPCA-based methods on ALLAML but becomes the least competitive on Imm40 and MSTAR\_SOC\_CNN. We can assume that the predefined graph used by LapScore is suitable in some applications while not be precise enough to describe the local structure of data in others. In comparison, SOGFS tends to be more stable on ten datasets because of the adaptive graph. However, the similarity matrix is still sensitive to noise features contained in real-world data. Therefore in general, spectral-based methods have less advantage in clustering task.
    \item \textbf{Convex SPCA-based methods generally outperform non-convex SPCA-based methods.} On PIE, Orlraws10P, ALLAML, Indoor\_Resnet50, Lung and Isolet, we can see that the curves of AW-SPCA-PSD, CSPCA-PSD and SPCA-PSD are always higher than that of SPCAFS. Just like we mention before, the adaptive weight matrix that measures the importance of each subspace component allows convex SPCA-based methods to reconstruct the data matrix better. There are, indeed, some applications (like Imm40 and MINIST) where SPCAFS can significantly outperform some of the convex methods. But the highest point belongs to at least one of the convex methods.
    \item \textbf{Convex SPCA-based methods with low-rank constraint perform better.} Due the lack of low-rank constraint, AW-SPCA-PSD can not always obtain a low-rank reconstruction matrix, which means it may not be able to project the data matrix into a low-dimensional manifold space. As a result, AW-SPCA-PSD sometimes has the least satisfying performance among three convex SPCA-based methods.
    \item \textbf{CSPCA-PSD and SPCA-PSD are both effective.} We can see that on most datasets, either CSPCA-PSD or SPCA-PSD obtains the highest ACC (or NMI). In other words, neither CSPCA-PSD nor SPCA-PSD can always perform better than the other. Since there is no evidence that $\ell_{2,1}$-norm necessarily outperforms Frobenious norm on reconstruction, the clustering results are reasonable. Anyway, CSPCA and SPCA-PSD are both competent on feature selection.   
\end{itemize}

\begin{figure*}[!t]
\centering
\includegraphics[width=7.5in]{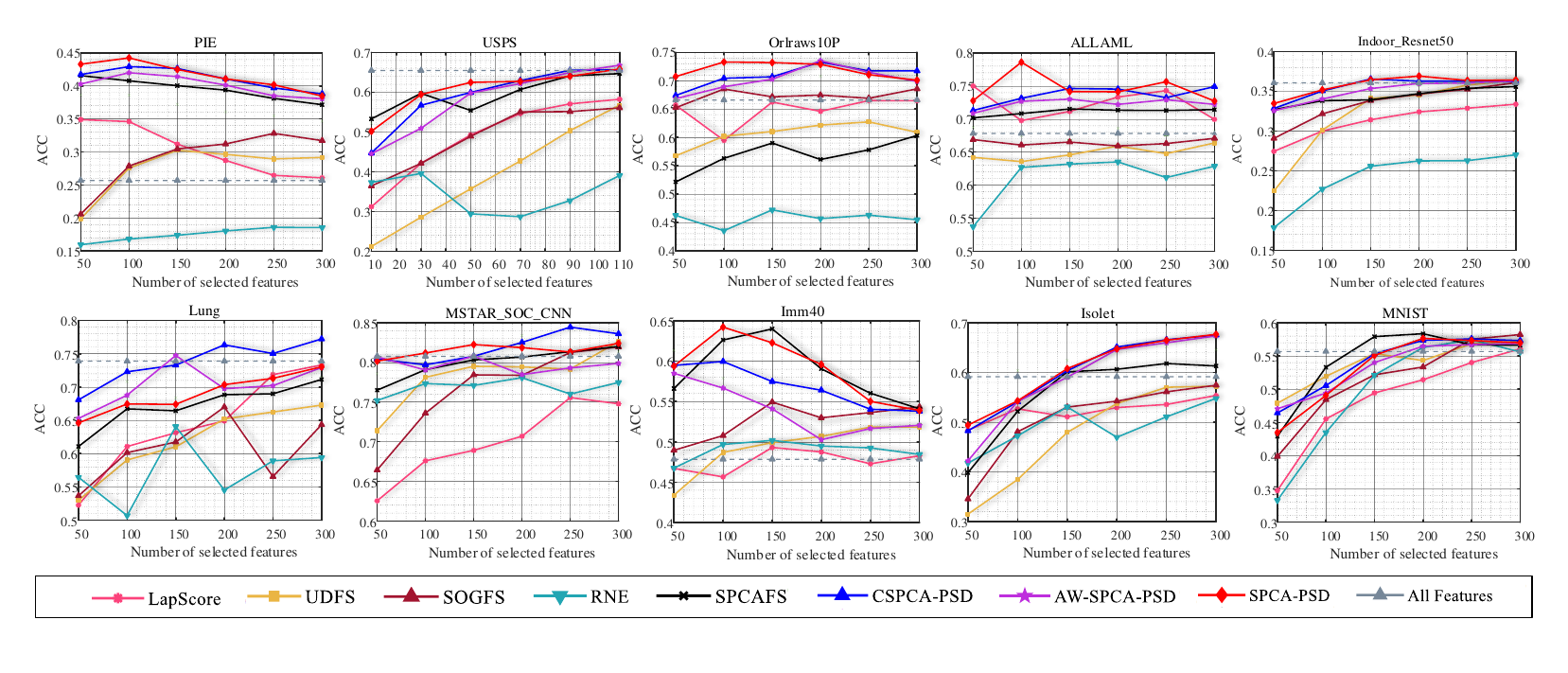}
\caption{ACC curves of different unsupervised feature selection methods on ten datasets. For each method, given the number of selected features, the best performance are recorded after conducting the algorithm with every possible regularization parameter combination from grid search. The red curves, blue curves and purple curves represent our proposed SPCA-PSD, CSPCA-PSD and AW-SPCA-PSD, respectively.}
\label{CFS_ACC}
\end{figure*}

\begin{figure*}[!t]
\centering
\includegraphics[width=7.5in]{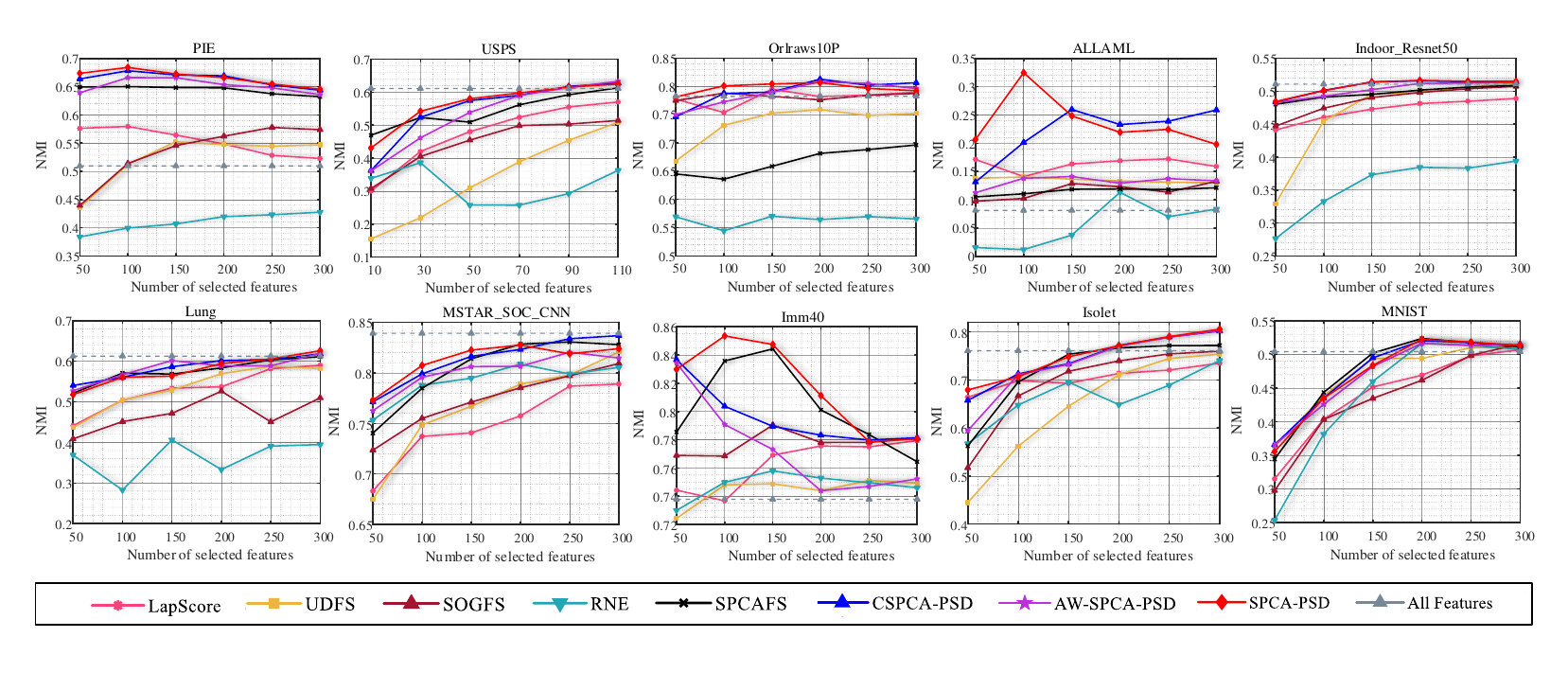}
\caption{NMI curves of different unsupervised feature selection methods on ten datasets. The experimental conditions are the same as they are in Fig.\ref{CFS_ACC}.}
\label{CFS_NMI}
\end{figure*}

\subsubsection{Visualization of feature selection} 
To visualize the effectiveness of SPCA-PSD, we compare its feature selection results on the face image dataset PIE with SPCAFS and SOGFS. The selected features are the best performing ones in terms of ACC from the clustering experiment. we set the number of selected features to be 100 and highlight these features in four randomly chosen original images. The results are shown in Table \ref{Visualization}. As is known, the features selected by a UFS method represent its 'understanding' of how classes are separated from each other considering all samples. We can see SPCAFS and SPCA-PSD tend to select more discriminative features: the basic characteristics of a human face (eyes, nose, mouth, lips), while SOGFS focuses on a limited region. Compared with SPCAFS, SPCA-PSD selects more diverse features as to maintain a more complete geometrical structure of a human face. When the number of selected features is relatively small, diverse distribution of selected features can make use of as many small regions of an image as possible, reducing redundant features. According to the clustering result in Fig. \ref{CFS_ACC}, SPCA-PSD dose achieve better performance than SPCAFS and SOGFS, thus effectiveness of the selected features can be proved. In conclusion, SPCA-PSD has a satisfying feature selection ability.

\begin{table}
\caption{Visualization of the selected features on PIE. The results are corresponding to the ACC curves in Fig.\ref{CFS_ACC}. We randomly choose four images from PIE and set the number of selected features to be 100. The yellow highlight pixels are the selected features. SPCA-PSD captures the most complete geometrical structure of a human face by selecting pixels of eyes, eyebrows, nose, mouth, bread, hair (on the top right) and glasses.} \label{Visualization}
\centering
\begin{tabular}{m{1.3cm}<{\centering} m{1.3cm}<{\centering} m{1.3cm}<{\centering} m{1.3cm}<{\centering} m{1.3cm}<{\centering}}
\toprule
  \textbf{Method}   &  \textbf{Image 1}   &
  \textbf{Image 2}    &
  \textbf{Image 3}    &
  \textbf{Image 4}
  \\
 \midrule
    Original &\includegraphics[width=0.5in]{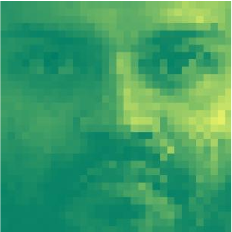} &\includegraphics[width=0.5in]{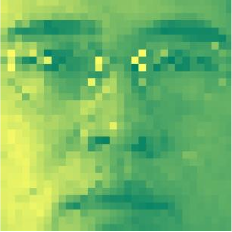}  &\includegraphics[width=0.5in]{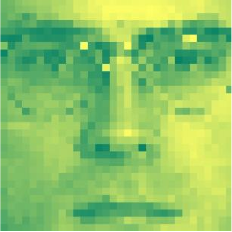}
    &\includegraphics[width=0.5in]{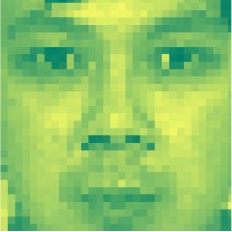} \\
    
     SOGFS &\includegraphics[width=0.5in]{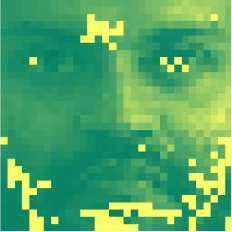} &\includegraphics[width=0.5in]{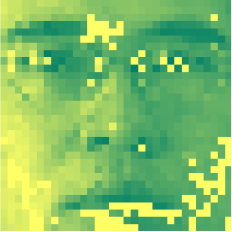}  
     &\includegraphics[width=0.5in]{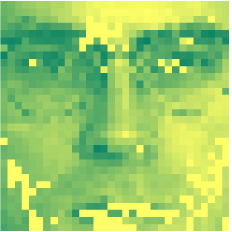} 
     &\includegraphics[width=0.5in]{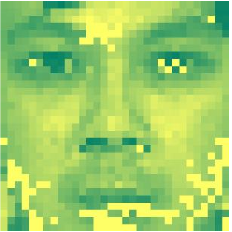}\\

     SPCAFS &\includegraphics[width=0.5in]{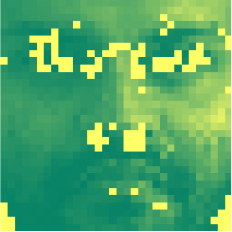} &\includegraphics[width=0.5in]{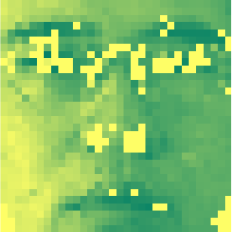}  
     &\includegraphics[width=0.5in]{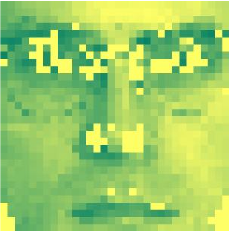}
     &\includegraphics[width=0.5in]{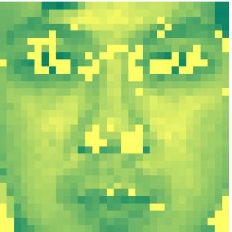}\\
     
    SPCA-PSD &\includegraphics[width=0.5in]{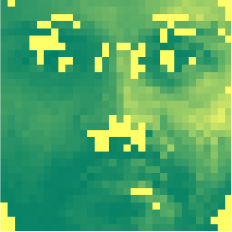} &\includegraphics[width=0.5in]{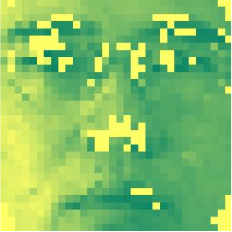}  
    &\includegraphics[width=0.5in]{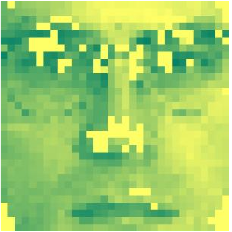} 
    &\includegraphics[width=0.5in]{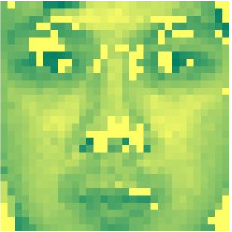} \\

 \bottomrule
\end{tabular}
\end{table}

\subsection{Experiment 4: Running Time Analysis}
In this section, we compare the running time of different EUFS methods to validate the efficiency of SPCA-PSD. We carry out an experiment on four data sets (USPS, MSTAR\_SOC\_CNN, Indoor\_Resnet50,  and MNIST) which have the largest number of samples and can best demonstrate the advantage of our proposed methods. Since LapScore is a filter method without optimization progress, we don't count it into comparison. We set all the regularization parameters to be fixed at 10, and record their running time. As for RNE, it doesn't have tunable regularization parameter, so we record the time when the number of selected features is fixed to 100. The experimental results are shown in Fig.\ref{RT}. We can see that SPCA-PSD takes less training time than other methods. Because the computational complexity of AW-SPCA-PSD and CSPCA-PSD is relevant to $n$ in each iteration, they cost more time than SPCA-PSD. As for SPCAFS, due to its non-convexity, sometimes it takes more iterations to converge, which means longer running time (e.g. Fig.\ref{RT}(a)). To prove this, we show the number of iterations of SPCA-PSD and SPCAFS for comparison in Table.\ref{SPCAFS and SPCA-PSD}. Generally, SPCA-PSD achieves faster computational speed with less number of iterations and less running time per iteration. In conclusion, SPCA-PSD is significantly fast compared to other competitive EUFS methods.

\begin{figure}[!t]
\centering
    \subfloat[MSTAR\_SOC\_CNN ($n=2425$)]{\includegraphics[width=1.6in]{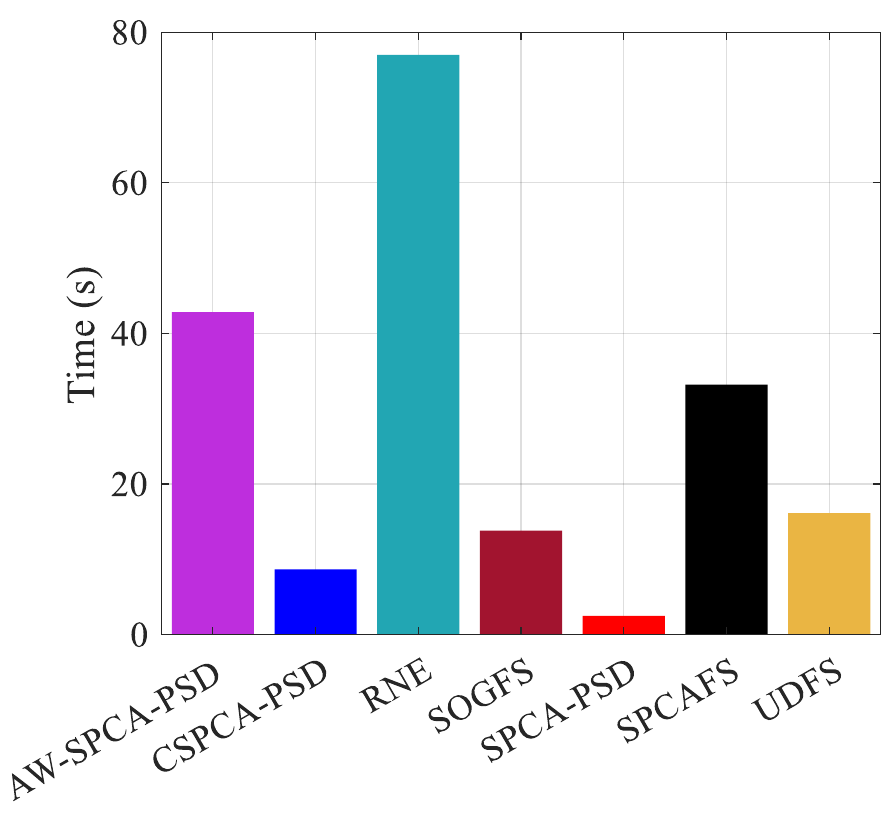}}\hspace{1em}
    \subfloat[USPS ($n=9298$)]{\includegraphics[width=1.6in]{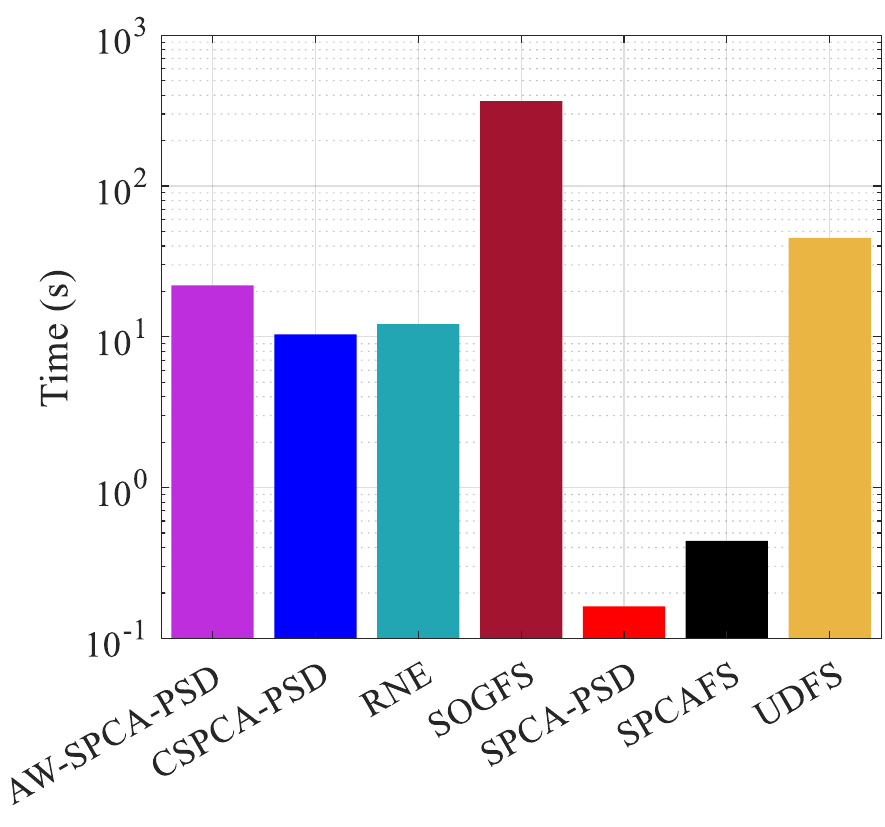}}\hspace{1em} \\
   \subfloat[Indoor\_Resnet50 ($n=15620$)]{\includegraphics[width=1.6in]{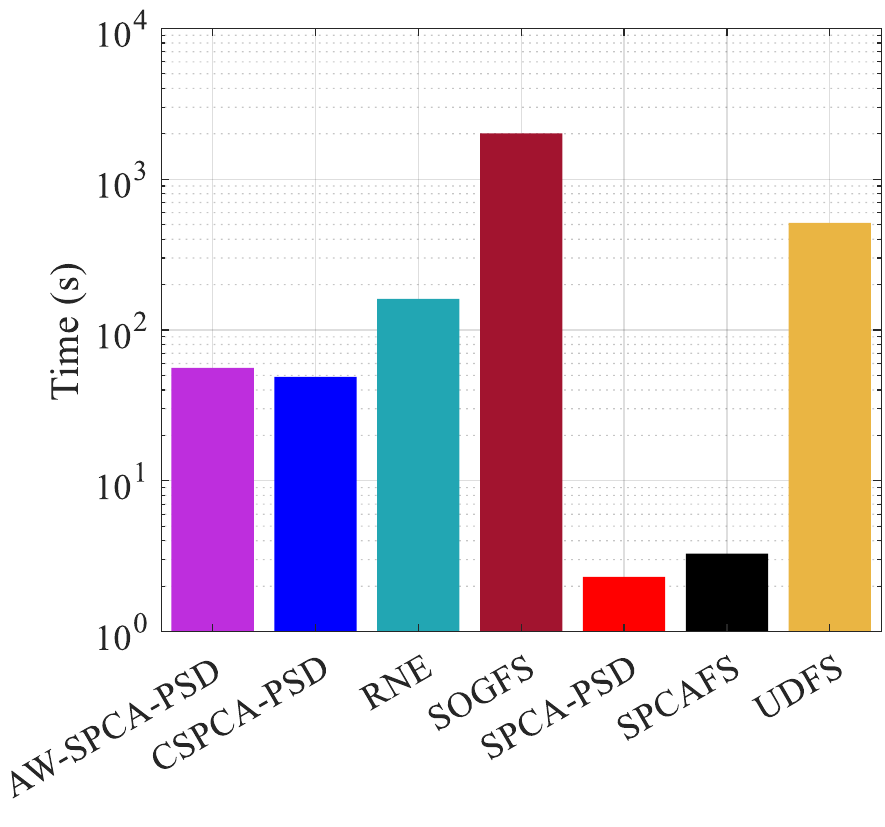}}
    \subfloat[MNIST ($n=20000$)]{\includegraphics[width=1.6in]{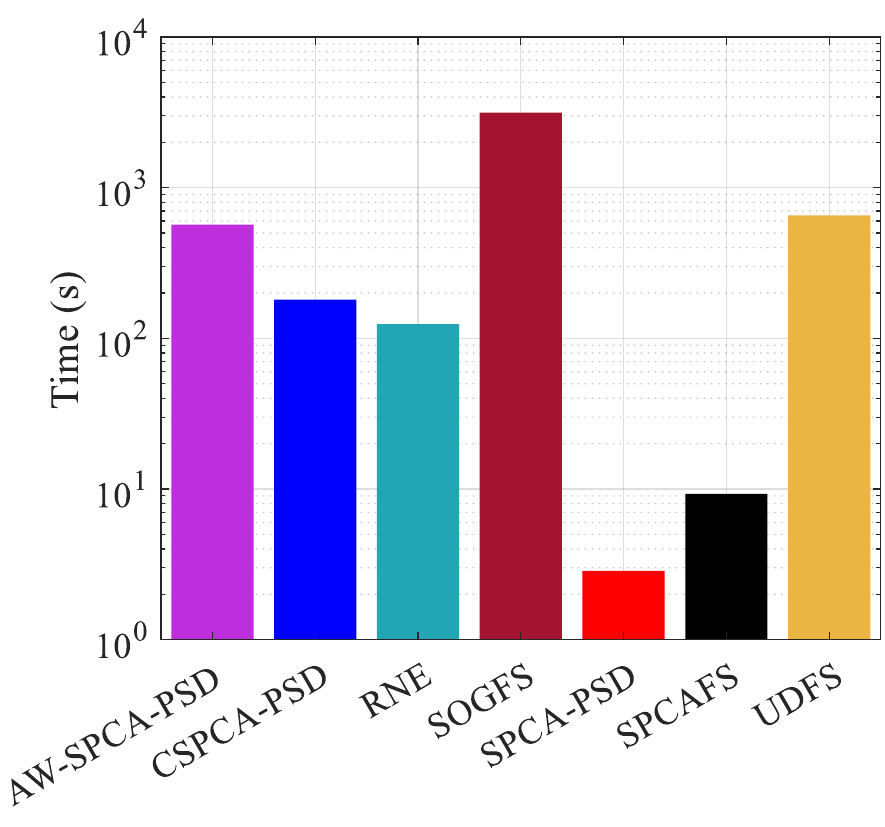}}
\caption{Running time on different data sets. $n$ denotes the number of samples in a data set. The time axis is linear in (a) and logarithmic in the rest. It can be observed that the running time of SPCA-PSD is affected the least by $n$.}
\label{RT}
\end{figure}

\begin{table}[!t] 
\caption{Number of iterations of SPCA-PSD and SPCAFS on different datasets (Number of Iterations , average running time per iteration (seconds)). It can seen that SPCA-PSD takes less number of iterations to converge than SPCAFS, except on Indoor\_Resnet50. And generally SPCA-PSD costs less time on each iteration than SPCAFS, too. }
\label{SPCAFS and SPCA-PSD}
\centering
\renewcommand\arraystretch{1.5}
\begin{tabular}{|c|c|c|}  \hline
\diagbox{Dataset}{Method}
& SPCA-PSD & SPCAFS
 \\ \hline
 MSTAR & (18, 0.1381) & (76, 0.4368)\\ \hline
 USPS  & (18, 0.0091) & (50, 0.0088)\\   \hline
 Indoor\_Resnet50  & (8, 0.2879)& (3, 1.0967) \\   \hline
 MNIST &(19, 0.1504)& (46, 0.2017) \\   \hline
\end{tabular}
\end{table}

\subsection{Experiment 5: Convergence study}
We have proven the convergence of \textbf{Algorithm} \textbf{\ref{Al1}} theoretically. Now we study the convergence in practice. We run SPCA-PSD with $\lambda$ and $\eta$ fixed to 10. The convergence curves of the objective function value on different data sets are shown in Fig.\ref{CA}. We show the results on PIE and Orlraw10P. It can be seen that the objective function generally decreases rapidly and converges within a small number of iterations. Although the convergence speed may vary in practice, the maximal number of iterations is usually under 50.

\begin{figure}[!t]
\begin{minipage}[b]{.5\linewidth}
    \centering
    \subfloat[PIE]{
    \includegraphics[width=1.7in]{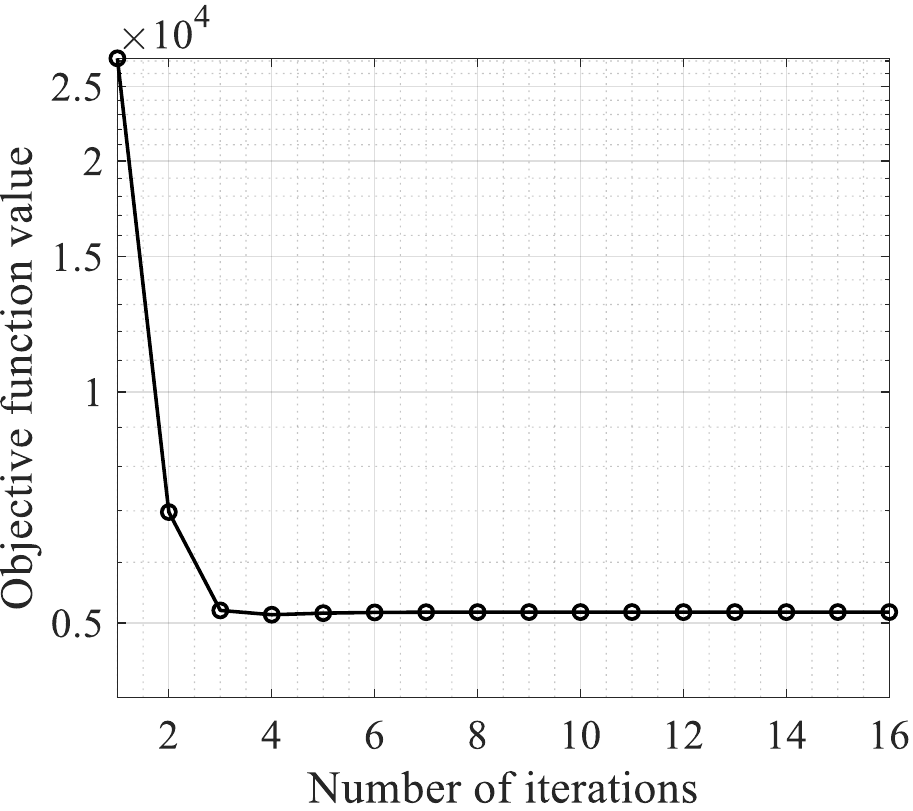}}
\end{minipage}
\begin{minipage}[b]{.5\linewidth}
    \centering
    \subfloat[Orlraws10p]{
    \includegraphics[width=1.7in]{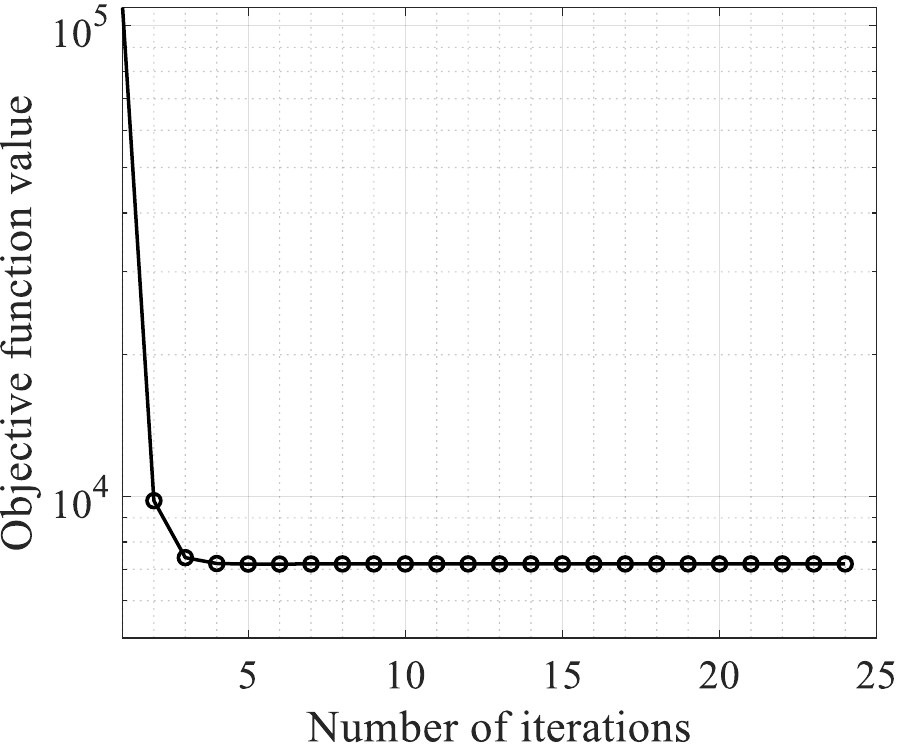}}
\end{minipage}
\caption{\centering{Convergence curves of SPCA-PSD on different data sets. The results are obtained by setting $\lambda$ and $\eta$ to be 10.}}
\label{CA}
\end{figure}

\subsection{Experiment 6: Parameter Sensitivity Analysis and Parameter-setting Strategy for SPCA-PSD}
In experiments, we find that the optimal setting of $\lambda$ and $\eta$ (in SPCA-PSD) is related to the variance of features $Tr(\mathbf{X}\mathbf{X}^T)$, where $\mathbf{X}$ is the data matrix. So we conduct several experiments to analyze parameter sensitivity and come out with a parameter-setting strategy for SPCA-PSD. We set the number of selected features to be 100 and try different combinations of $\lambda$ and $\eta$ in the range of $\left\{10^{-4},10^{-3},10^{-2},10^{-1},1,10,10^{2},10^{3},10^{4}\right\}$. We only show the clustering ACC results on Isolet, ALLAML, Imm40 and Lung, but the idea is applicable to other data sets. 

Fig.\ref{Isolet and ALLAML} shows that the ACC bars on four datasets all reach a peak when $\lambda$ or $\eta$ is set to be in the interval of $\left[1\%Tr(\mathbf{X}\mathbf{X}^T),10\%Tr(\mathbf{X}\mathbf{X}^T\right]$. For SPCA-PSD, $\lambda$ and $\eta$ control the sparsity and the rank of the reconstruction matrix separately. The sparser the reconstruction matrix is, the more redundant features it will abandon, until the sparsity reaches a point where even important features are excluded. That explains why ACC on Isolet (or Imm40) drops fast when the value of $\lambda$ crosses some threshold. As for ALLAML (or Lung), it has a larger $Tr(\mathbf{X}\mathbf{X}^T)$ than Isolet. So naturally the peak of ACC on ALLAML remains when $\lambda$ is greater than $10^3$. And a similar analysis can be done to describe how $\eta$ affects the performance. The low-rank property of the reconstruction matrix is directly concerned with the global manifold learning ability of SPCA. Therefore in certain range, the larger $\eta$ there is, the better. However, when the rank is so low that the reconstruction matrix can no longer preserve most information in data matrix, the performance could have a serious lost.

Although it seems that either fixing $\lambda$ or $\eta$ to be in the range of  $\left[1\%Tr(\mathbf{X}\mathbf{X}^T),10\%Tr(\mathbf{X}\mathbf{X}^T\right]$ can lead to the best performance of SPCA-PSD, we recommend a priority of adjusting $\eta$, considering the running time. We conduct another experiment on Isolet and ALLAML to prove this idea. We still set the number of selected features to be 100. And we fix $\lambda$ (or $\eta$) to compare the running time and ACC of SPCA-PSD with different $\eta$ (or $\lambda$). Note that we fix $\lambda$ (or $\eta$) to be the value that makes ACC reach the peak in Fig.\ref{Isolet and ALLAML}. The results are shown in Fig.\ref{eta and lambda}. We observe that for both Isolet and ALLAML, the running time of SPCA-PSD with fixed $\eta$ is often less than that with fixed $\lambda$, while the clustering ACC are on the same level. When $\lambda$ is equal to or larger than $10\%\eta$, the running time could rise sharply, thus we also recommend to set $\lambda$ at the value no greater than $10\%\eta$. 

We conclude our parameter-setting strategy for SPCA-PSD as: given a data matrix $\mathbf{X}\in\mathbb{R}^{d\times{n}}$, 1) set $\eta$ to be in range of $\left[1\%Tr(\mathbf{X}\mathbf{X}^T),10\%Tr(\mathbf{X}\mathbf{X}^T\right]$, 2) set $\lambda$ to be no larger than $10\%\eta$. With this strategy, SPCA-PSD can be both effective and efficient. As for CSPCA-PSD and AW-SPCA-PSD, similar experiments can be conducted to find their suitable parameter-setting strategy. 

\begin{figure}[!t]
\centering
   \subfloat[Isolet ($Tr(\mathbf{X}\mathbf{X}^T)=10^{5.2423}$)]{\includegraphics[width=1.6in]{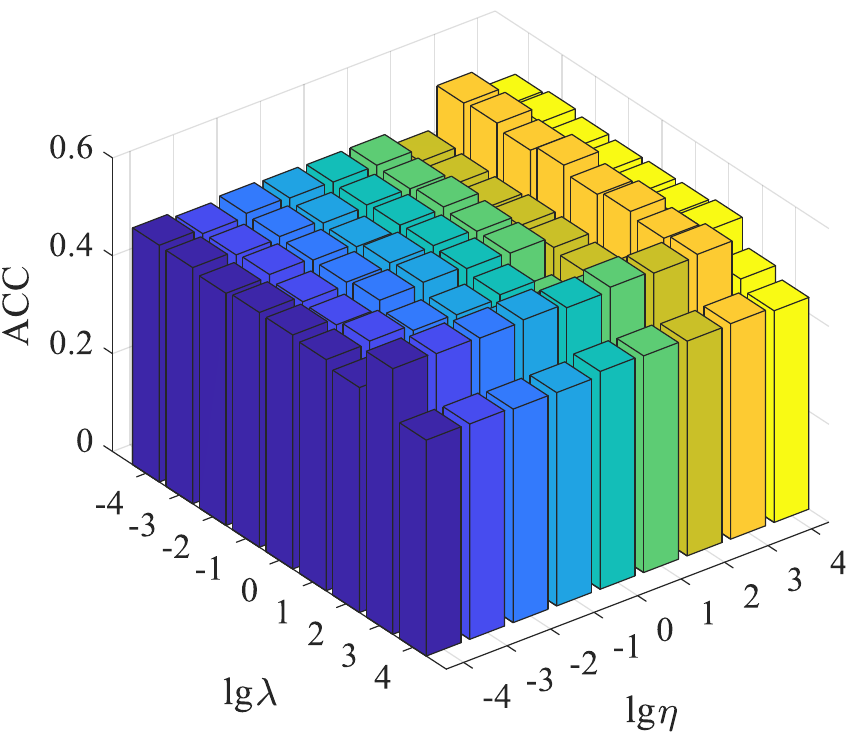}}\hspace{1em}
   \subfloat[ALLAML ($Tr(\mathbf{X}\mathbf{X}^T)=10^{5.7043}$)]{\includegraphics[width=1.6in]{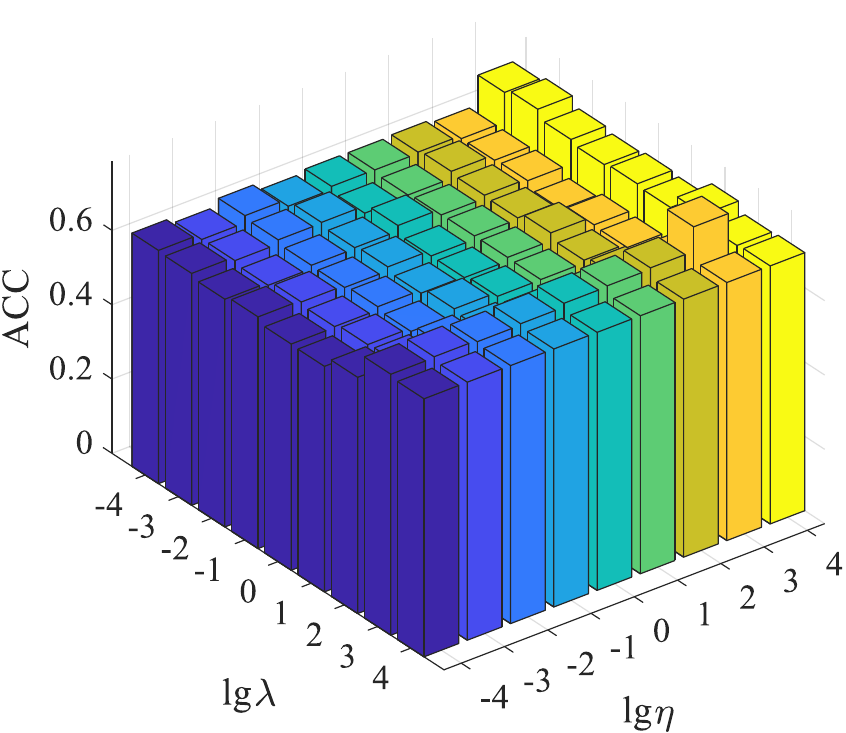}}\\
    \subfloat[Imm40 ($Tr(\mathbf{X}\mathbf{X}^T)=10^{3.4807}$)]{\includegraphics[width=1.6in]{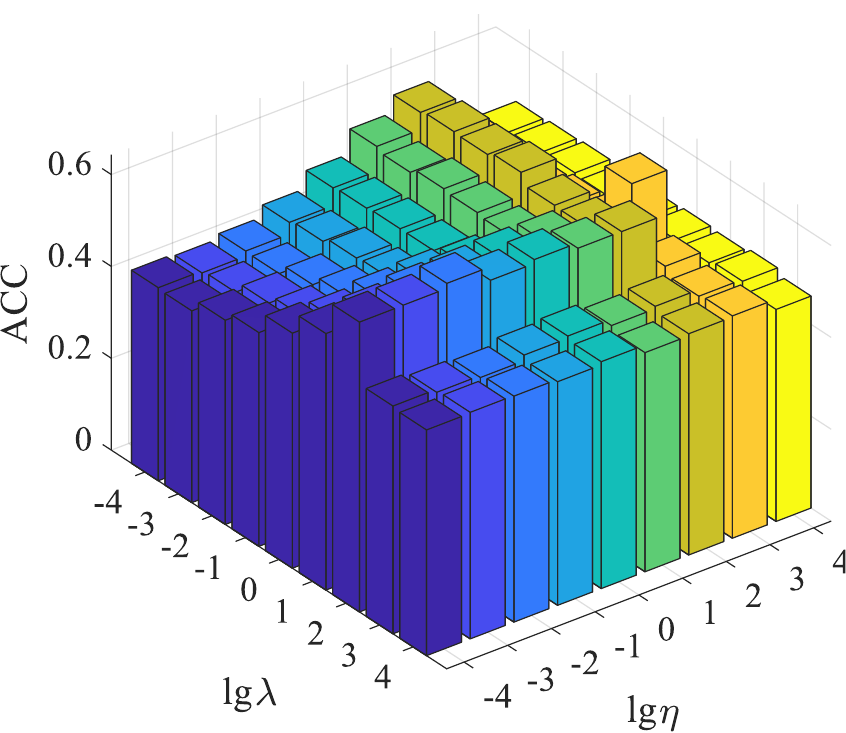}}\hspace{1em}
    \subfloat[Lung \protect\\($Tr(\mathbf{X}\mathbf{X}^T)=10^{5.7405}$)]{\includegraphics[width=1.6in]{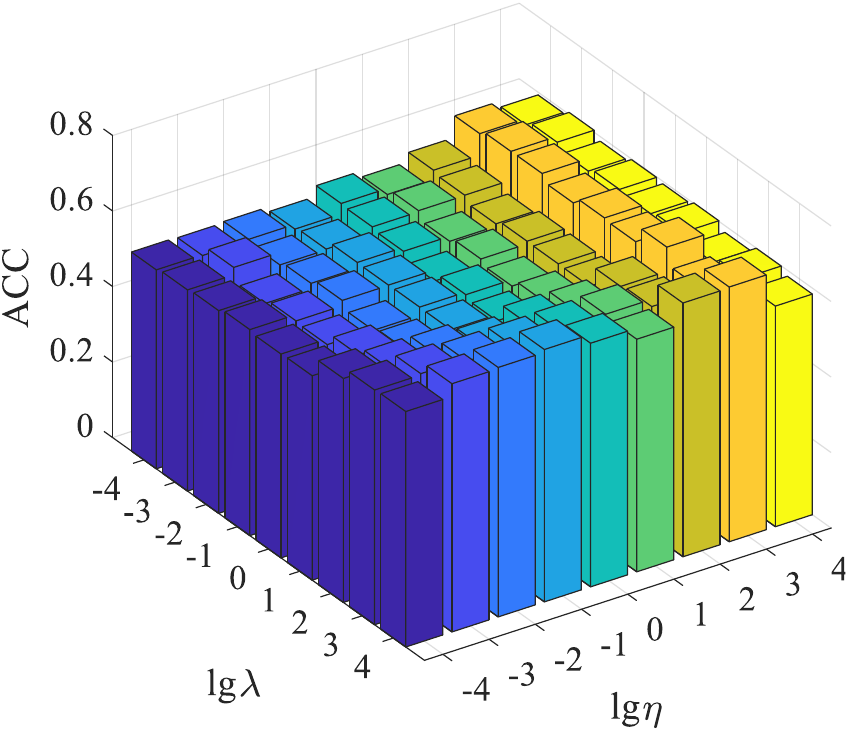}}
\caption{Clustering ACC on Isolet and ALLAML with $\lambda$ and $\eta$ varying. The number of selected features is fixed to 100. lg$\lambda$ and lg$\eta$ denote $\log_{10}{\lambda}$ and $\log_{10}{\eta}$ respectively. The positions corresponding to the best performance share a common pattern that is related to $Tr(\mathbf{X}\mathbf{X}^T$ determined by the data matrix.}
\label{Isolet and ALLAML}
\end{figure}

\begin{figure}[!t]
\centering
   \subfloat[Isolet ($\lambda=10^3$)]{\includegraphics[width=1.7in]{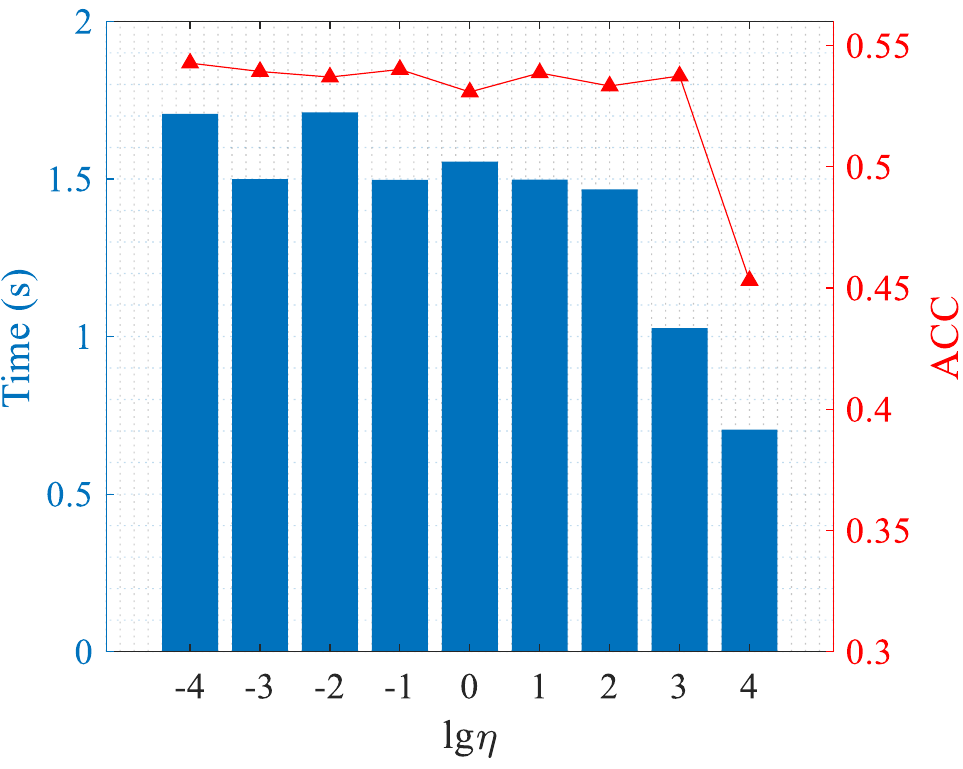}}\hspace{1em}
   \subfloat[Isolet ($\eta=10^3$)]{\includegraphics[width=1.7in]{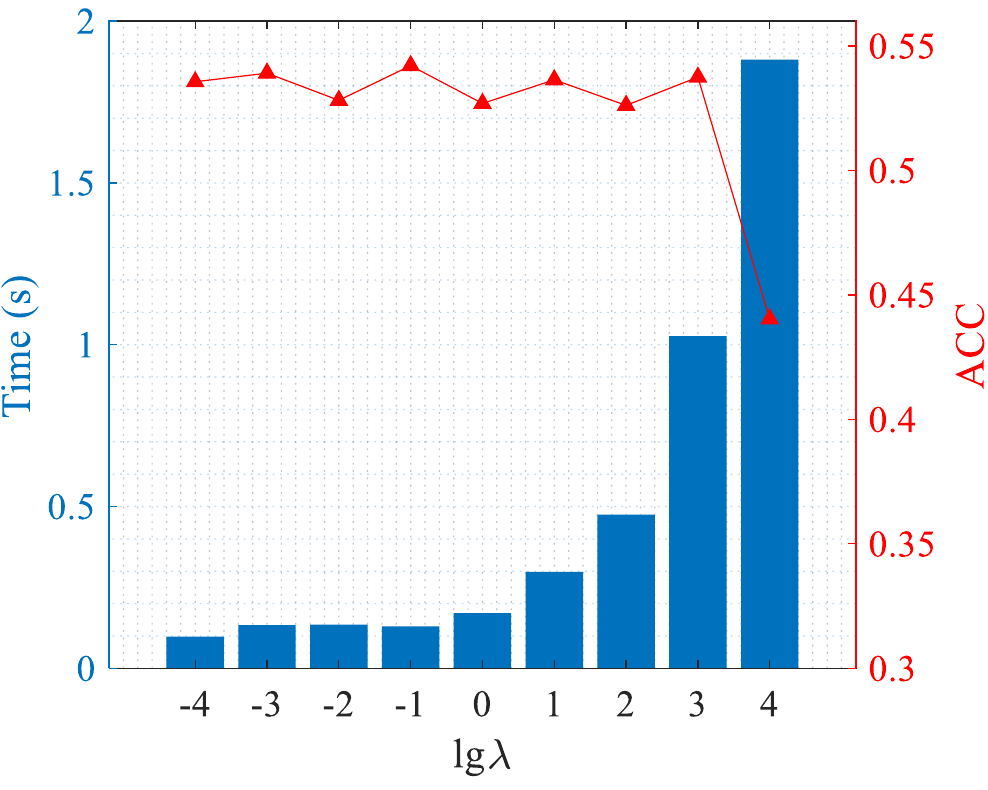}}\\
   \subfloat[ALLAML ($\lambda=10^4$)]{\includegraphics[width=1.7in]{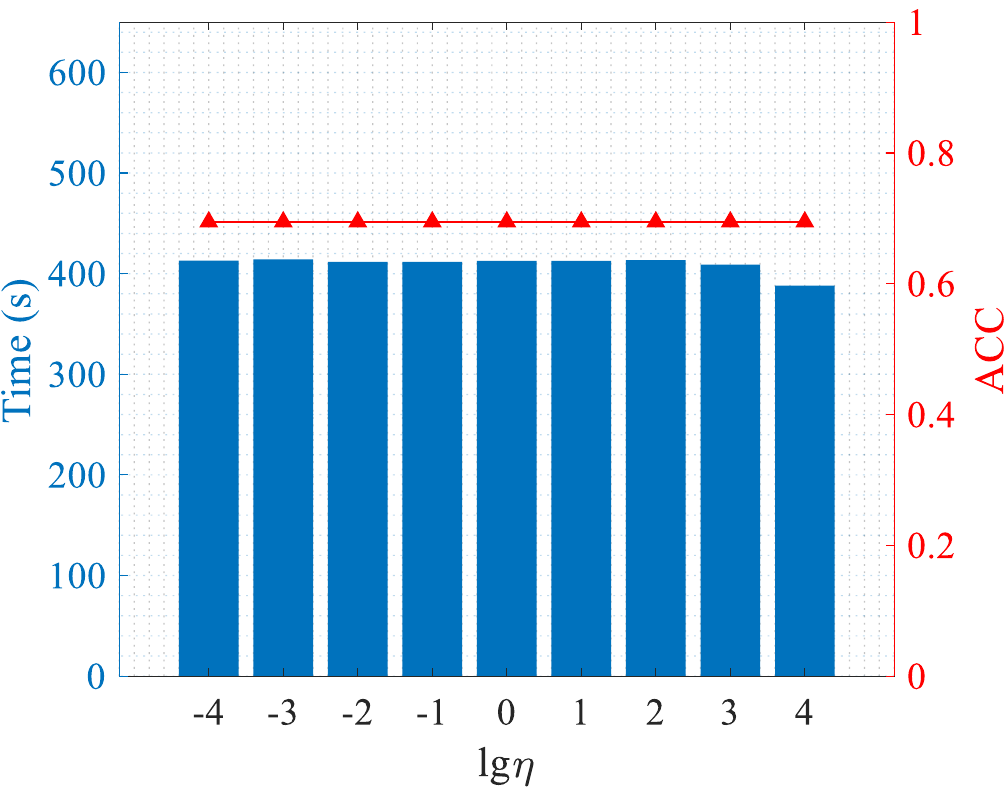}}\hspace{1em}
   \subfloat[ALLAML ($\eta=10^4$)]{\includegraphics[width=1.7in]{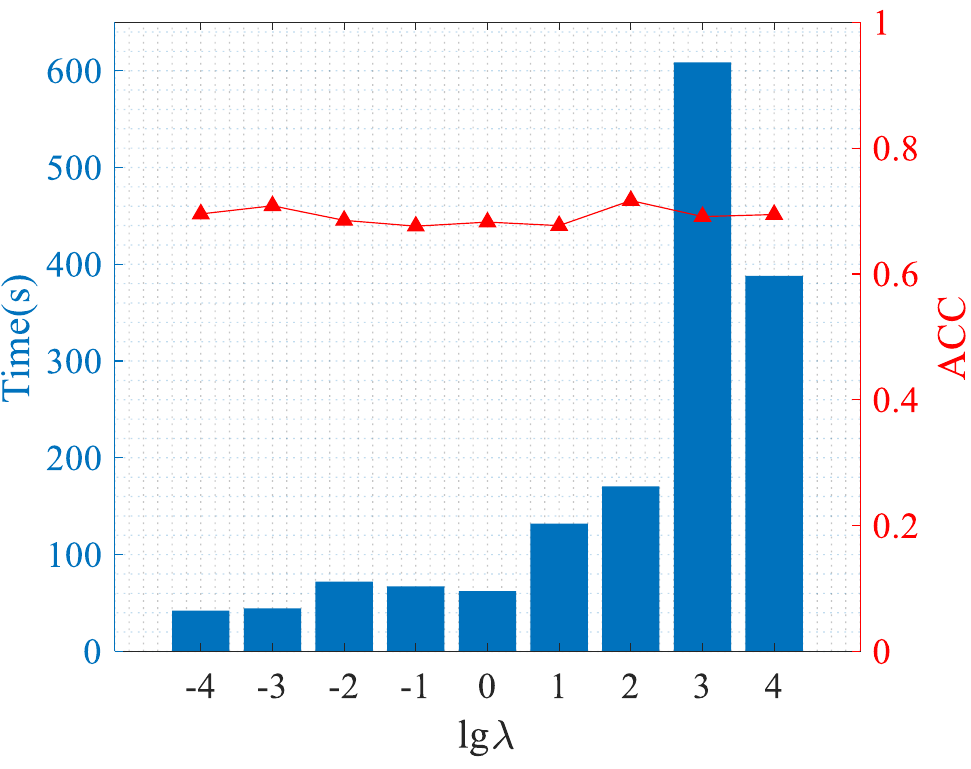}}
\caption{Running time bars and ACC curves on Isolet, ALLAML, Imm40 and Lung with $\lambda$ (or $\eta$) fixed and $\eta$ (or $\lambda$) varying. To achieve similar performance, it is often time-saving to fix $\eta$ in the range of $1\%Tr(\mathbf{X}\mathbf{X}^T)-10\%Tr(\mathbf{X}\mathbf{X}^T)$ and change $\lambda$.}
\label{eta and lambda}
\end{figure}

\section{Conclusions}
In this paper, we propose a standard convex SPCA-based model for unsupervised feature selection, and design a two-step fast optimization algorithm. We reformulate SPCA as a convex model and prove that the optimal solution falls onto the PSD cone. By adopting PSD projection in the optimization algorithm, SPCA-PSD achieves both effectiveness and efficiency. The optimal solution of the reconstruction matrix is used to select discriminative features. We prove the convergence of the proposed algorithm. We also prove that for other existing convex SPCA-based EUFS methods CSPCA and AW-SPCA, the PSD constraint holds. Therefore we propose the PSD versions of them. Experiments on both synthetic and real-world data sets are conducted to demonstrate the effectiveness and efficiency of SPCA-PSD, CSPCA-PSD and AW-SPCA-PSD. We also analyze the parameter sensitivity of SPCA-PSD, and provide a parameter-setting strategy that helps to obtain the best performance while costing the least time.

\appendices
\section{Proof of the PSD constraint in CSPCA and AW-SPCA}
\begin{theorem} 
Let $\mathbf{U}^*$ and $\mathbf{Q}^*$ be the optimal solution to the following optimization problem, Then $ \mathbf{U}^*(\mathbf{Q}^*)^T\in{S^d_+} $.
\end{theorem}
\begin{equation}
    \begin{aligned} \label{Problem 31}
    \min\limits_{\mathbf{U},\mathbf{Q}} \quad &{\Vert {\mathbf{X}-\mathbf{U}\mathbf{Q}^T\mathbf{X}} \Vert}_{2,1}+\lambda{\Vert \mathbf{Q}^T \Vert}_{2,1}\\
    s.t. \quad&{\mathbf{U}^T\mathbf{U}=\mathbf{I}_k}
\end{aligned}
\end{equation}
\begin{proof}
We can rewrite Problem (\ref{Problem 31}) as 
\begin{equation} \label{Problem 32}
\begin{aligned}
 &\min\quad{\left\Vert {\mathbf{X}-\sum\limits_{j=1}^{k}\mathbf{u}_j\mathbf{q}_j^T\mathbf{X}}\right\Vert}_{2,1}+\lambda\sum\limits_{j=1}^{k}{\Vert {\mathbf{q}_j}\Vert}_2\\
&s.t.\quad \mathbf{u}_j^T\mathbf{u}_j=1
\end{aligned}
\end{equation}
Then the objective function can be expanded as 
\begin{equation} \label{Problem 33}
\begin{aligned}
&{\left\Vert {\mathbf{X}-\sum\limits_{j=1}^{k}\mathbf{u}_j\mathbf{q}_j^T\mathbf{X}}\right\Vert}_{2,1}+\lambda\sum\limits_{j=1}^{k}{\Vert {\mathbf{q}_j}\Vert}_2\\
&={\left\Vert\left({\mathbf{X}-\sum\limits_{j=1}^{k}\mathbf{u}_j\mathbf{q}_j^T\mathbf{X}}\right)\sqrt{\mathbf{W}}\right\Vert}^2_F+\lambda\sum\limits_{j=1}^{k}{\Vert {\mathbf{q}_j}\Vert}_2\\
&=Tr(\mathbf{X}\mathbf{W}\mathbf{X}^T)-\sum\limits_{j=1}^{k}\bigg[2Tr(\mathbf{u}_j^T\mathbf{X}\mathbf{W}\mathbf{X}^T\mathbf{q}_j)\\
&-Tr(\mathbf{q}_j^T\mathbf{X}\mathbf{W}\mathbf{X}^T\mathbf{q}_j)\bigg.-\left.\frac{\lambda}{\Vert \mathbf{q}_j\Vert_2}\mathbf{q}_j^T\mathbf{q}_j\right]\\
&=Tr(\mathbf{X}\mathbf{W}\mathbf{X}^T)-\sum\limits_{j=1}^{k}\bigg[2(\mathbf{u}_j^T\mathbf{X}\mathbf{W}\mathbf{X}^T\mathbf{q}_j)-(\mathbf{q}_j^T\mathbf{X}\mathbf{W}\mathbf{X}^T\mathbf{q}_j)\bigg.\\
&-\left.\frac{\lambda}{\Vert \mathbf{q}_j\Vert_2})\mathbf{q}_j^T\mathbf{q}_j\right]
\end{aligned}
\end{equation}
where $\mathbf{W}\in{\mathbb{R}^{n\times{n}}}$ is a diagonal matrix whose $i$-th element is $\left(1/\left(2\left\Vert \left[\mathbf{X}-\sum\limits_{j=1}^{k}\mathbf{u}_j\mathbf{q}_j^T\mathbf{X}\right]_j\right\Vert_2\right)\right)$.\par
If we view (\ref{Problem 33}) as a sum of $k$ subproblems with respect to $\mathbf{u}_j$ and $\mathbf{q}_j$, then given a fixed $\mathbf{W}$ and a fixed $\mathbf{u}_j$, we can have each subproblem minimized at
\begin{equation} \label{Problem 34}
   \mathbf{q}_j^*=\left(\mathbf{X}\mathbf{W}\mathbf{X}^T+\frac{\lambda}{\Vert \mathbf{q}_j\Vert_2}\right)^{-1}\mathbf{X}\mathbf{W}\mathbf{X}^T\mathbf{u}_j
\end{equation}
Substitude (\ref{Problem 34}) into (\ref{Problem 33}) and we have 
\begin{equation}
    \mathbf{u}_j^*={\underset{\mathbf{u}_j^T\mathbf{u}_j=1}{\arg\min}}\mathbf{u}^T\mathbf{X}\mathbf{W}\mathbf{X}^T\left(\mathbf{X}\mathbf{W}\mathbf{X}^T+\frac{\lambda}{\Vert \mathbf{q}_j \Vert_2}\right)^{-1}\mathbf{X}\mathbf{W}\mathbf{X}^T\mathbf{u}_j
\end{equation}\par
It can be solved by performing an eigenvalue decomposition: $\mathbf{X}\mathbf{W}\mathbf{X}^T\left(\mathbf{X}\mathbf{W}\mathbf{X}^T+\frac{\lambda}{\Vert \mathbf{q}_j \Vert_2}\right)^{-1}\mathbf{X}\mathbf{W}\mathbf{X}^T=\mathbf{V}\mathbf{\Sigma}\mathbf{V}^T$. Hence $\mathbf{u}_j^*=s_j\mathbf{v}_j$ with $s_j=1$ or $-1$. Then, we obtain $\mathbf{q}_j^*=s_j\frac{\sigma^2_{jj}}{\sigma^2_{jj}+{\lambda}/{\Vert \mathbf{q}_j\Vert_2}}\mathbf{v}_j$. Therefore, we have
\begin{equation}\label{Cls 2}
    \mathbf{U}^*{(\mathbf{Q}^*)}^T=\sum\limits_{j=1}^{k}\mathbf{u}_j^*(\mathbf{q}^*)^T_j=\mathbf{V}\mathbf{D}\mathbf{V}^T\in{S^d_+}
\end{equation}
where $d_{jj}=\frac{\sigma^2_{jj}}{\sigma^2_{jj}+{\lambda}/{\Vert \mathbf{q}_j\Vert_2}} \geq 0$. For each iteration, given fixed $\mathbf{W}$, $\mathbf{U}^*{(\mathbf{Q}^*)}^T$ satisfies (\ref{Cls 2}).Thus, (\ref{Cls 2}) still holds when Problem (\ref{Problem 31}) is solved. \par
Further, if we let $\mathbf{\Omega}=\mathbf{U}^*{(\mathbf{Q}^*)}^T$, then  ${\Vert \mathbf{Q}^T \Vert_{2,1}}={\Vert \mathbf{U}\mathbf{Q}^T \Vert_{2,1}}$. And a trace norm can be used to transform $rank(\mathbf{\Omega})\leq k$ into a regularization term. Hence, we obtain the following problem 
\begin{equation}\label{Problem 37}
\begin{aligned}
    \min\quad &{\Vert {\mathbf{X}-\mathbf{\Omega}\mathbf{X}} \Vert}_{2,1}+\lambda{\Vert \mathbf{\Omega} \Vert}_{2,1}+\eta{Tr\left(\left(\mathbf{\Omega}\mathbf{\Omega}^T\right)^\frac{1}{2}\right)}\\
    s.t. \quad&{\mathbf{\Omega}\in{S^d_+}}\\
\end{aligned}
\end{equation}
which is in fact CSPCA with PSD constraint. And if we replace $\mathbf{X}$ with $\mathbf{X}-\mathbf{b}\mathbf{1}^T$, where $\mathbf{b}\in{\mathbf{R}^{d\times{1}}}$ is the mean vector of $\mathbf{X}$ in the $\ell_{2,1}$-norm space, similar deduction can be done to prove that $\mathbf{\Omega}\in{S^d_+}$ still holds. Then we yield AW-SPCA with PSD constraint.\par
In conclusion, the optimal solution of the reconstruction matrix in either CSPCA or AW-SPCA falls onto the PSD cone. 
\end{proof}

%
\ifCLASSOPTIONcompsoc
\else
  \section*{Acknowledgment}
\fi

\ifCLASSOPTIONcaptionsoff
  \newpage
\fi



 \bibliographystyle{IEEEtran}
 \bibliography{Reference}
 \vspace{-1cm}

%


%
\begin{IEEEbiography}[{\includegraphics[width=1in,height=1.25in,clip,keepaspectratio]{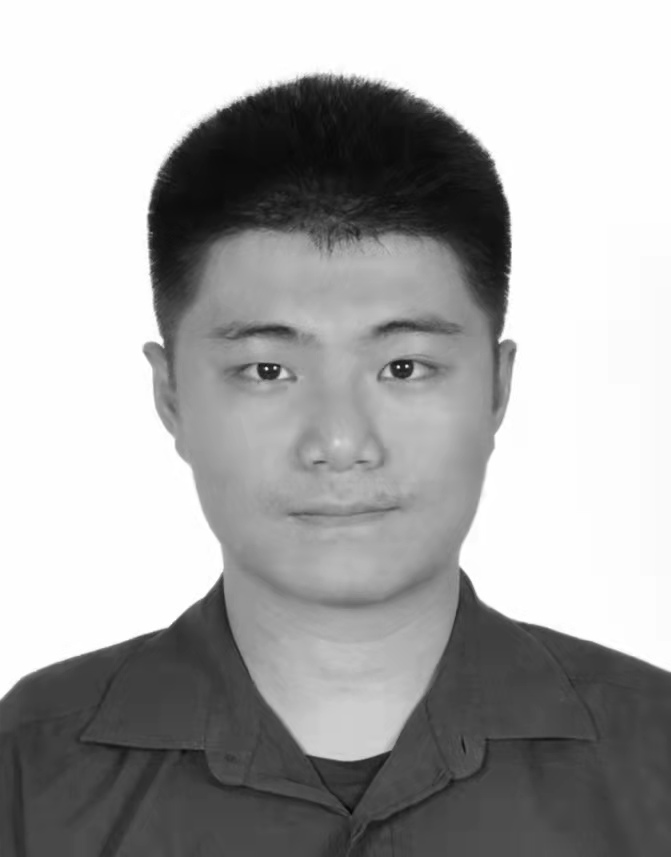}}]{Junjing Zheng}
received the B.E degree from the Central South University of China(CSU), Changsha, in 2021. He is currently pursuing the Ph.D degree in information and communication engineering with the National University of Defense Technology(NUDT), Changsha.
His research interests include machine learning, pattern analysis, and target recognition.
\end{IEEEbiography}
\vspace{-1cm}
\begin{IEEEbiography}[{\includegraphics[width=1in,height=1.25in,clip,keepaspectratio]{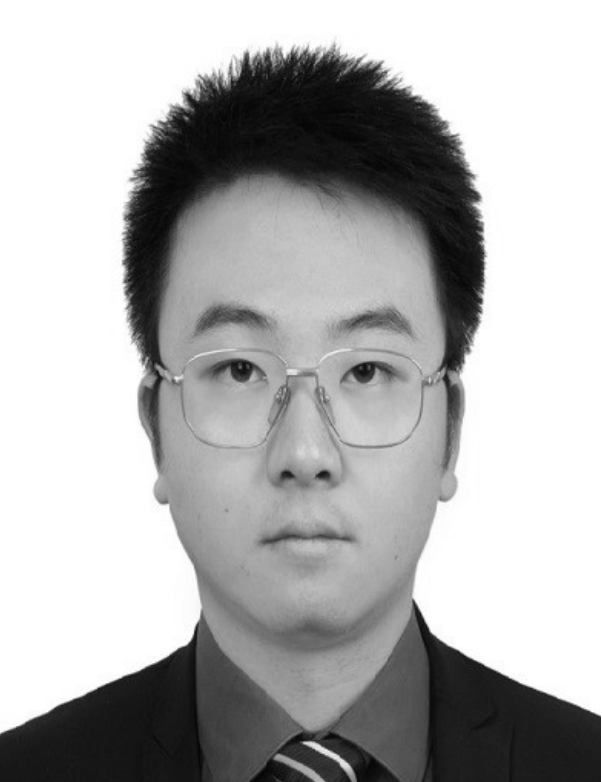}}]{Xinyu Zhang}
received the B.S. and Ph.D.
degrees from the Beijing Institute of Technology,
Beijing, China, in 2011 and 2017, respectively.
From 2015 to 2017, he visited the Ohio State
University as a Visiting Scholar. Since 2017,
he has been holding a postdoctoral position with
the National University of Defense Technology.
He is currently a Lecturer with the National
University of Defense Technology. His research
interests include array signal processing, auto
target detection, and waveform optimization.
\end{IEEEbiography}
\vspace{-1cm}
\begin{IEEEbiography}[{\includegraphics[width=1in,height=1.5in,clip,keepaspectratio]{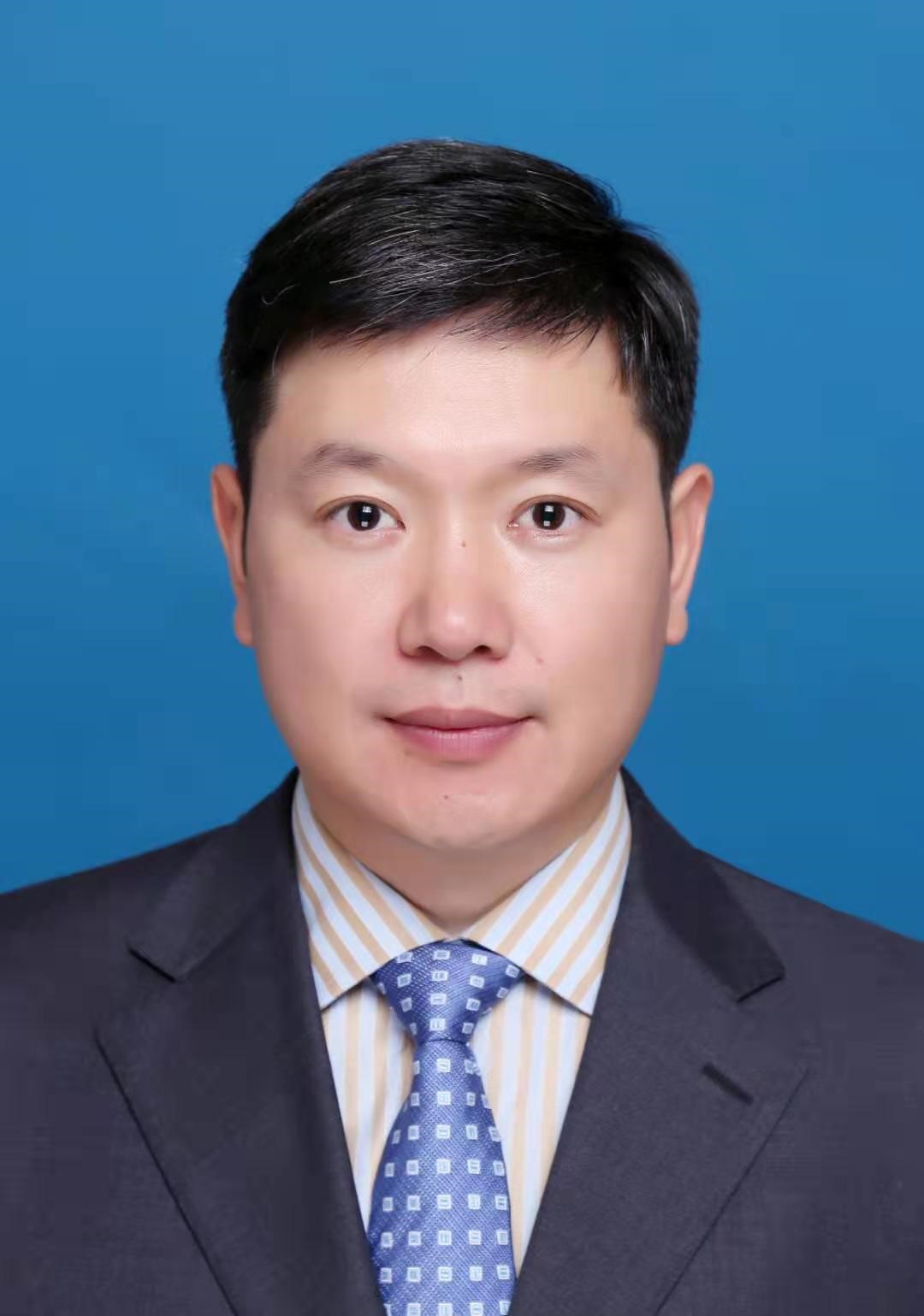}}]{Yongxiang Liu}
received his Ph.D. degree in Information and Communication Engineering from National University of Defense Technology (NUDT), Changsha, China, in 2004. Currently, He is a Full Professor in the College of Electronic Science and Technology, National University of Defense Technology. His research interests mainly include remote sensing imagery analysis, radar signal processing, Synthetic Aperture Radar (SAR) object recognition and Inverse SAR (ISAR) imaging, and machine learning.
\end{IEEEbiography}
\vspace{-1cm}
\begin{IEEEbiography}[{\includegraphics[width=1in,height=1.25in,clip,keepaspectratio]{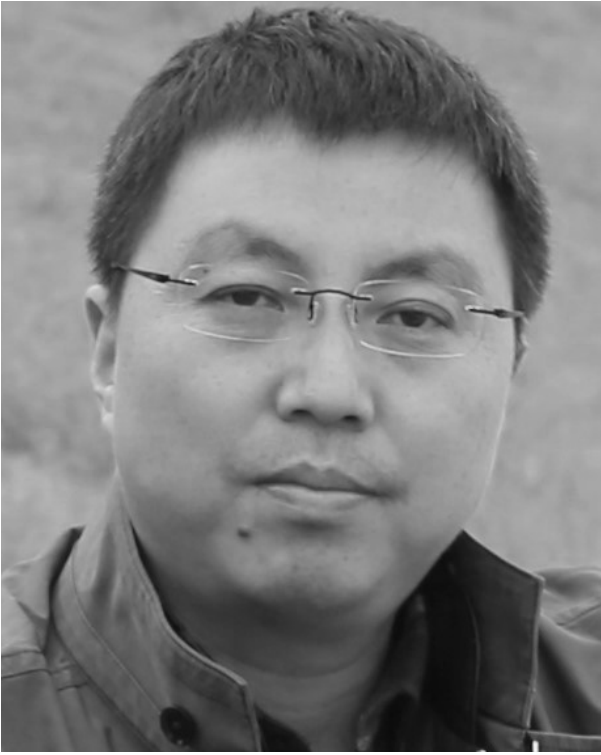}}]{Weidong Jiang}
was born in Chongqing,
China, in 1968. He received the B.S. degree
in communication engineering and the Ph.D.
degree in electronic science and technology
from the National University of Defense Technology (NUDT), China, in 1991 and 2001,
respectively.
He is currently a Professor with NUDT. His
current research interests include multiple-input
multiple-output radar signal processing and radar
system technology.
\end{IEEEbiography}
\vspace{-1cm}
\begin{IEEEbiography}[{\includegraphics[width=1in,height=1.25in,clip,keepaspectratio]{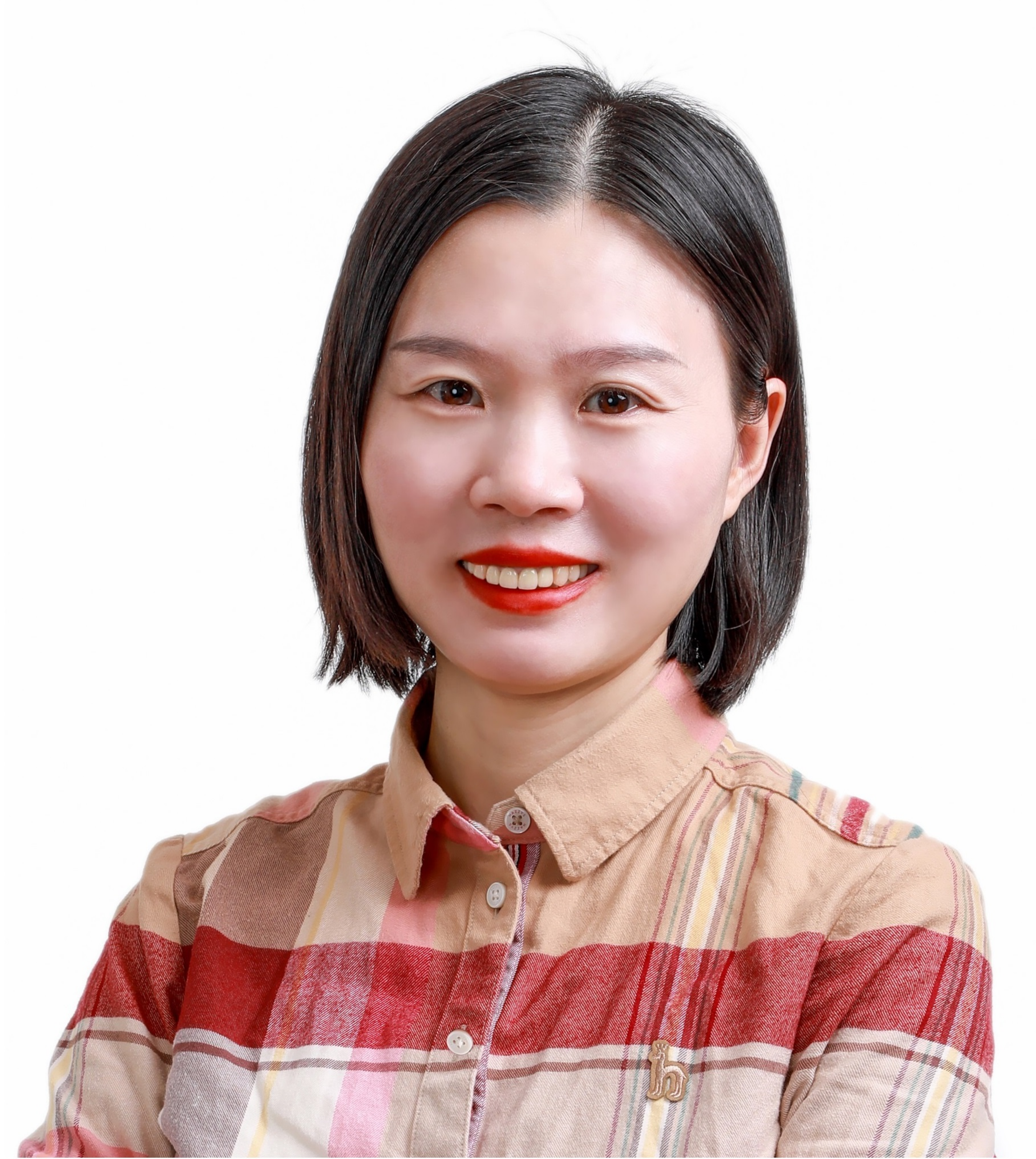}}]{Li Liu} received her Ph.D. degree from the National University
of Defense Technology (NUDT), China, in 2012. During her PhD
study, she spent two years as a Visiting
Student at the University of Waterloo. From 2015 to 2016, she spent ten months visiting
the Multimedia Laboratory at the Chinese University of
Hong Kong. From 2016 to 2018, she was a senior researcher of the CMVS at the University of Oulu, Finland. Dr. Liu
served as a cochair of many International Workshops along with major venues like CVPR and ICCV.
 She served as the leading guest editor of the special issues for IEEE TPAMI and
IJCV. She also served as an Area Chair for several respected international conferences. She currently serves as an Associate Editor for IEEE TGRS, IEEE TCSVT, and Pattern Recognition. Her research interests include computer vision, pattern recognition, and machine learning.
Her papers currently have 9000+ citations according to Google Scholar.
\end{IEEEbiography}
\vspace{-1cm}
\begin{IEEEbiography}[{\includegraphics[width=1in,height=1.25in,clip,keepaspectratio]{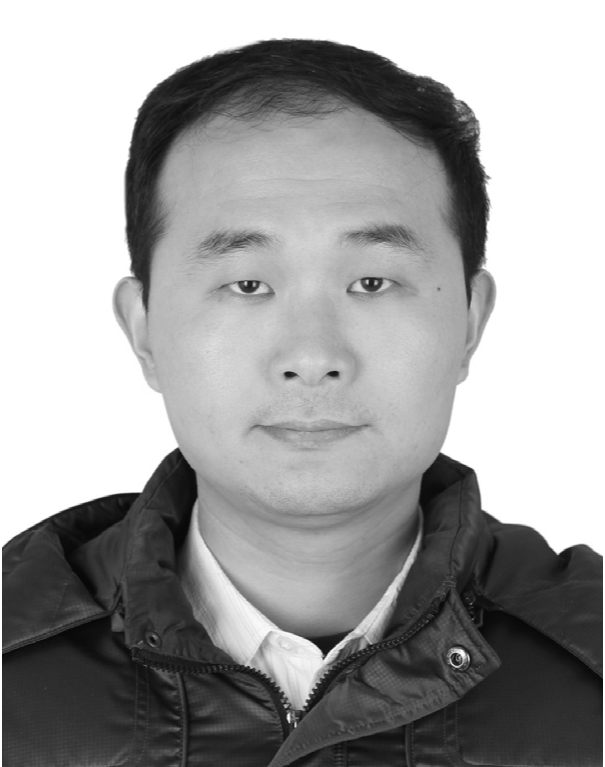}}]{Kai Huo}
was born in Chongqing,
China, in 1968. He received the B.S. degree
in communication engineering and the Ph.D.
degree in electronic science and technology
from the National University of Defense Technology (NUDT), China, in 1991 and 2001,
respectively.
He is currently a Professor with NUDT. His
current research interests include multiple-input
multiple-output radar signal processing and radar
system technology.
\end{IEEEbiography}





\end{document}